\pdfoutput=1

\documentclass[accepted]{uai2021} 

\usepackage[american]{babel}

\usepackage{natbib} 
\usepackage{mathtools} 
\usepackage{booktabs} 
\usepackage{tikz} 
\usepackage{microtype}
\usepackage{graphicx}
\usepackage{subfigure}
\usepackage{booktabs} 
\usepackage{algorithm}
\usepackage{algorithmic}

\usepackage{amsmath,amsfonts,amsthm,bm}

\def\eqref#1{equation~\ref{#1}}

\def\1{\bm{1}}

\DeclareMathAlphabet{\mathsfit}{\encodingdefault}{\sfdefault}{m}{sl}
\SetMathAlphabet{\mathsfit}{bold}{\encodingdefault}{\sfdefault}{bx}{n}

\newcommand{\bx}{{\mathbf{x}}}

\newcommand{\bz}{{\mathbf{z}}}
\newcommand{\bt}{{\mathbf{t}}}

\newcommand{\Dc}{{\mathcal{D}}}

\newtheorem{corollary}{Corollary}

\newtheorem{lemma}{Lemma}

\usepackage{url}
\usepackage{graphicx}
\usepackage{multicol}
\usepackage{multirow}
\usepackage{booktabs}

\usepackage{wrapfig}
\usepackage{float}

\usepackage{pifont}
\usepackage{xcolor}
\usepackage{tabularx}
\usepackage{caption}
\usepackage{wrapfig}
\usepackage{fixltx2e}
\usepackage{thmtools,thm-restate}

\newcommand{\dre}{\texttt{DRE}}
\makeatletter

\usepackage{hyperref}
\usepackage{cleveref}[2012/02/15]

\title{Featurized Density Ratio Estimation}

\author[1]{\href{mailto:<kechoi@cs.stanford.edu>?Subject=Your UAI 2021 paper}{Kristy~Choi\thanks{Denotes equal contribution.}}{}} 
\author[1]{\href{mailto:<madelineliao@stanford.edu>?Subject=Your UAI 2021 paper}{Madeline~Liao$^*$}{}}
\author[1]{\href{mailto:<ermon@cs.stanford.edu>?Subject=Your UAI 2021 paper}{Stefano~Ermon}{}}
\affil[1]{%
    Computer Science Department\\
    Stanford University
}
\begin{document}
\maketitle

\begin{abstract}

Density ratio estimation serves as an important technique in the unsupervised machine learning toolbox. However, such ratios are difficult to estimate for complex, high-dimensional data, particularly when the densities of interest are sufficiently different. In our work, we propose to leverage an invertible generative model to map the two distributions into a common feature space prior to estimation. 
This featurization brings the densities closer together in latent space, sidestepping pathological scenarios where the learned density ratios in input space can be arbitrarily inaccurate. At the same time,
the invertibility of our feature map guarantees that the ratios computed in feature space are equivalent to those in input space. 
Empirically, we demonstrate the efficacy of our approach in a variety of downstream tasks that require access to accurate density ratios such as mutual information estimation, targeted sampling in deep generative models, and classification with data augmentation. 
\end{abstract}
\section{Introduction}
\label{intro}

\begin{figure}[!ht]
\centering
\includegraphics[width=.4\textwidth]{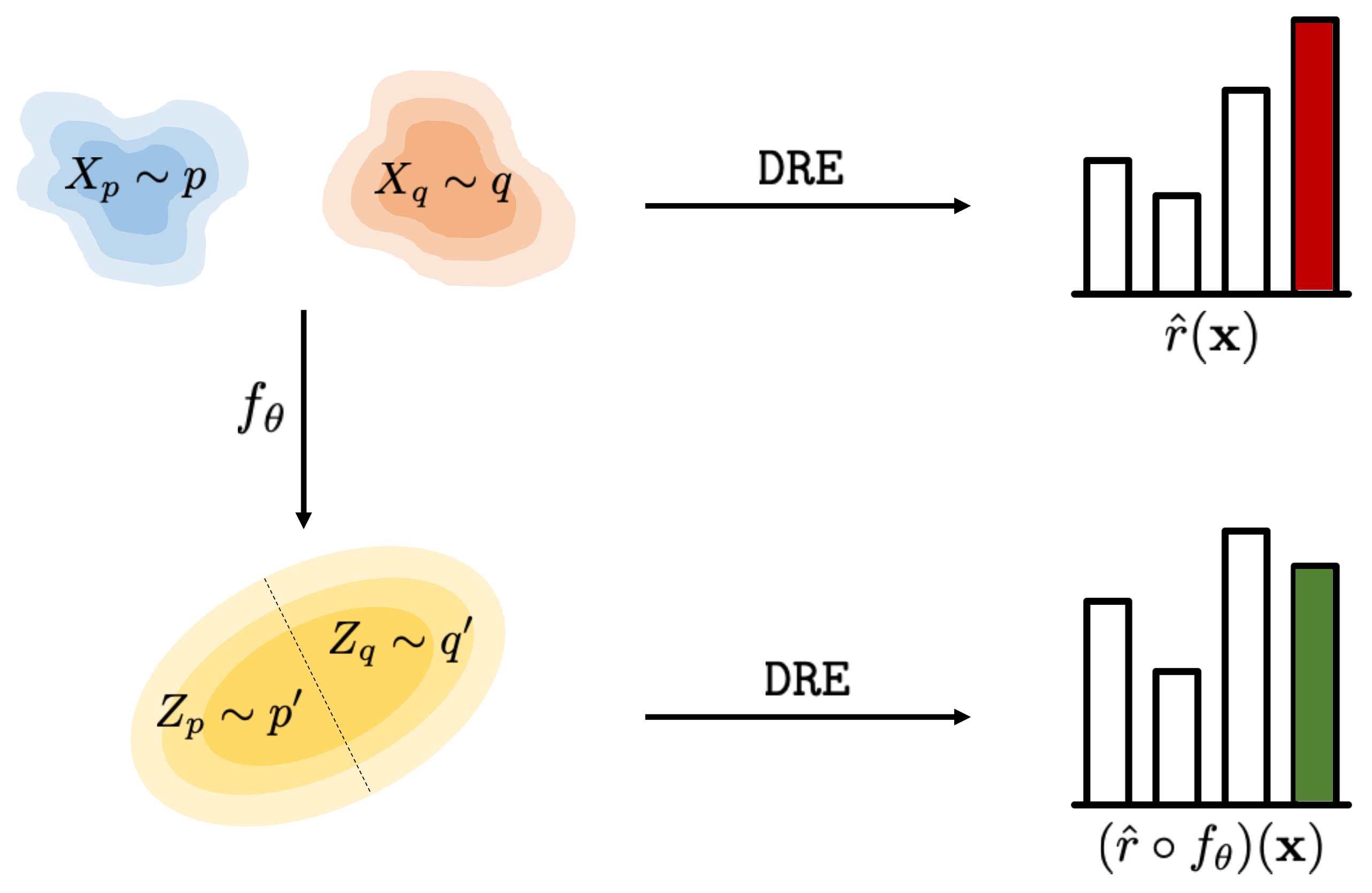}
\caption{Flowchart for the featurized density ratio estimation framework. Direct density ratio estimation using a black-box algorithm \texttt{DRE} on samples leads to poor ratio estimates $\hat{r}(\bx)$ when $p$ and $q$ are sufficiently different. By training a normalizing flow $f_\theta$ on samples from both densities and encoding them to a shared feature space prior to estimation, we obtain more accurate ratios $(\hat{r} \circ f_\theta)(\bx)$.}
\label{fig:model_flowchart}
\end{figure}

A central problem in unsupervised machine learning is that of density ratio estimation: given two sets of samples drawn from their respective data distributions, we desire an estimate of the ratio of their probability densities \citep{nguyen2007estimating,sugiyama2012density}. Computing this ratio gives us the ability to \textit{compare and contrast} two distributions, and is of critical importance in settings such as out-of-distribution detection \citep{smola2009relative,menon2016linking}, mutual information estimation \citep{belghazi2018mutual,song2019understanding}, importance weighting under covariate shift \citep{huang2006correcting,gretton2009covariate,you2019towards}, and hypothesis testing \citep{gretton2012kernel}. Related areas of research which require access to accurate density ratios, such as generative modeling \citep{gutmann2010noise,goodfellow2014generative,nowozin2016f} and unsupervised representation learning \citep{thomas2021likelihood}, have enjoyed tremendous success with the development of more sophisticated techniques for density ratio estimation. 

Despite its successes, density ratio estimation is an extremely hard problem when the two distributions of interest are considerably different \citep{cortes2010learning,yamada2013relative,rhodes2020telescoping}. 
The fundamental challenge in reliably estimating density ratios in this scenario is precisely the access to only a finite number of samples. As the distance between the densities increases, we become less likely to observe samples that lie in low-density regions between the two distributions. Therefore, without an impractically large training set, our learning algorithm is highly likely to converge to a poor estimator of the true underlying ratio.


To address this challenge, we propose a general-purpose framework for improved density ratio estimation that brings the two distributions closer together in latent space.
The key component of our approach is an invertible generative model (normalizing flow), which is trained on a mixture of datasets drawn from the two distributions and used to map samples into a shared feature space prior to ratio estimation \citep{rezende2015variational}.
Encoding the data via the normalizing flow transforms the observed samples from the two densities to lie within a unit Gaussian ball.
We observe that this contraction helps mitigate pathological scenarios where the learned ratio estimates are wildly inaccurate. 
The invertibility of our feature map then guarantees that the ratios computed in feature space are equivalent to those in input space.
We demonstrate the generality of our framework by pairing it with several existing density ratio estimation techniques, and explore various training procedures in estimation algorithms that require learning a probabilistic classifier. A flowchart of our featurized density ratio estimation algorithm can be found on Figure~\ref{fig:model_flowchart}.

Empirically, we evaluate the efficacy of our approach on downstream tasks that require access to accurate density ratios. First, we demonstrate that applying our approach to existing density ratio estimation techniques on synthetic data leads to better performance on downstream domain adaptation and mutual information (MI) estimation.
Next, we demonstrate the utility of our framework on a targeted generation task on MNIST \citep{lecun1998mnist}. By leveraging the ``featurized" density ratios for importance sampling from a trained generative model, we show that the resulting samples are closer to the target distribution of interest than the synthetic examples generated using input-space density ratios.
Finally, we illustrate that our method can be used to improve upon naive data augmentation methods by reweighing synthetic samples, outperforming relevant baselines on multi-class classification on Omniglot \citep{lake2015human}. 

The contributions of our work can be summarized as:
\begin{enumerate}
    \item We introduce a general-purpose algorithm for estimating density ratios in feature space and show its applicability to a suite of existing ratio estimation techniques.
    \item By leveraging the invertibility of our feature map, we prove that our featurized density ratio estimator inherits key properties such as unbiasedness and consistency from the original ratio estimation algorithm.
    \item On downstream tasks that require access to accurate density ratios, we show that our approach outperforms relevant baselines that compute
    ratios in input space. 
\end{enumerate}
\section{Preliminaries}
\label{prelims}
\subsection{Invertible Transformations via Normalizing Flows}
Deep invertible generative models, or normalizing flows, are a family of likelihood-based models that describe the two-way transformation between a complex, continuous probability density and a simple one by 
the change of variables formula \citep{rezende2015variational,papamakarios2019normalizing}. 
The flow is parameterized by a deep neural network $f_\theta: \mathcal{X} \rightarrow \mathcal{Z}$ with carefully designed architectures such that the overall transformation is composed of a series of bijective mappings with tractable inverses and Jacobian determinants \citep{dinh2016density,kingma2016improving,papamakarios2017masked,kingma2018glow,grathwohl2018ffjord,ho2019flow++}.
As a result, the density of the random variable $X = f_\theta^{-1}(Z)$
can be evaluated \textit{exactly}:
\begin{align*}
    p(\bx) = t(f_\theta(\bx)) \left\vert \textrm{det} \frac{\partial f_\theta(\bx)}{\partial \bx}\right\vert
\end{align*}
where $f_\theta^{-1}: \mathcal{Z} \rightarrow \mathcal{X}$ denotes the inverse of the mapping $f_\theta$, 
$p$ denotes the probability density of the random variable $X$, and $t$ denotes the probability density of $Z$. 
The base (prior) distribution $t(\bz)$ is typically chosen to be an isotropic Gaussian $\mathcal{N}(0,I)$ -- the simplicity of evaluating this prior density, coupled with the tractability of $f_\theta^{-1}$ and its Jacobian, allows us to train the normalizing flow via maximum likelihood. 
The key property of normalizing flows that we exploit in our method is their \textit{invertibility}: the dimensionality of $\mathcal{X}$ and $\mathcal{Z}$ are the same by design, and data points can be losslessly mapped between the two spaces. As we will demonstrate in Section~\ref{method}, this will be critical for translating the density ratios obtained in latent space back to those in input space. 

\subsection{Density Ratio Estimation Techniques}
\label{existing-dre}
\paragraph{Notation and Problem Setup.}

We denote the input variable as $\bx \in \mathcal{X} \subseteq \mathbb{R}^d$, and let $\bz \in \mathcal{Z} \subseteq \mathbb{R}^d$ be a latent variable of the same dimensionality as the input. 
We use capital letters to denote random variables, e.g. $X_p \sim p$ and $X_q \sim q$.
Then, $Z_p = f_\theta(X_p)$ and $Z_q = f_\theta(X_q)$ denote the random variables obtained by transforming $X_p$ and $X_q$ with $f_\theta$.
We note that since $Z_p$ and $Z_q$ are transformed by the same normalizing flow $f_\theta$, we can form the mixture density $Z = \frac{1}{2}Z_p + \frac{1}{2}Z_q \sim t$, where $t(\bz) \sim \mathcal{N}(0,I)$.

The learning setting we consider is as follows. Given two sets of observed samples $\mathcal{D}_p = \{\bx_i^p\}_{i=1}^{n_p} \sim p(\bx)$ and $\mathcal{D}_q = \{\bx_j^q\}_{j=1}^{n_q} \sim q(\bx)$, we wish to estimate the ratio of their underlying probability densities $r(\bx) = p(\bx)/q(\bx)$. 
We focus on \textit{direct} ratio estimation techniques, where we learn the density ratio estimator $\hat{r}$ rather than constructing explicit \textit{density estimates} of $\hat{p}(\bx)$ and $\hat{q}(\bx)$ and computing their ratio \citep{sugiyama2012density}.
The estimator $\hat{r}$ is obtained via an 
estimation 
algorithm $\texttt{DRE}$ which takes as input two datasets and returns a function $\texttt{DRE}(\Dc_p, \Dc_q) = \hat{r}: \mathcal{X} \rightarrow \mathbb{R}$.
Then, evaluating $\hat{r}$ at a particular point $\bx$ gives us an estimate of the true density ratio $\hat{r}(\bx) \approx r(\bx)$. In the following exposition, we provide background information on the suite of existing density ratio estimation algorithms.


\begin{figure*}[!t]
    \centering
        \subfigure[Baseline classifier]{\includegraphics[width=.24\textwidth]{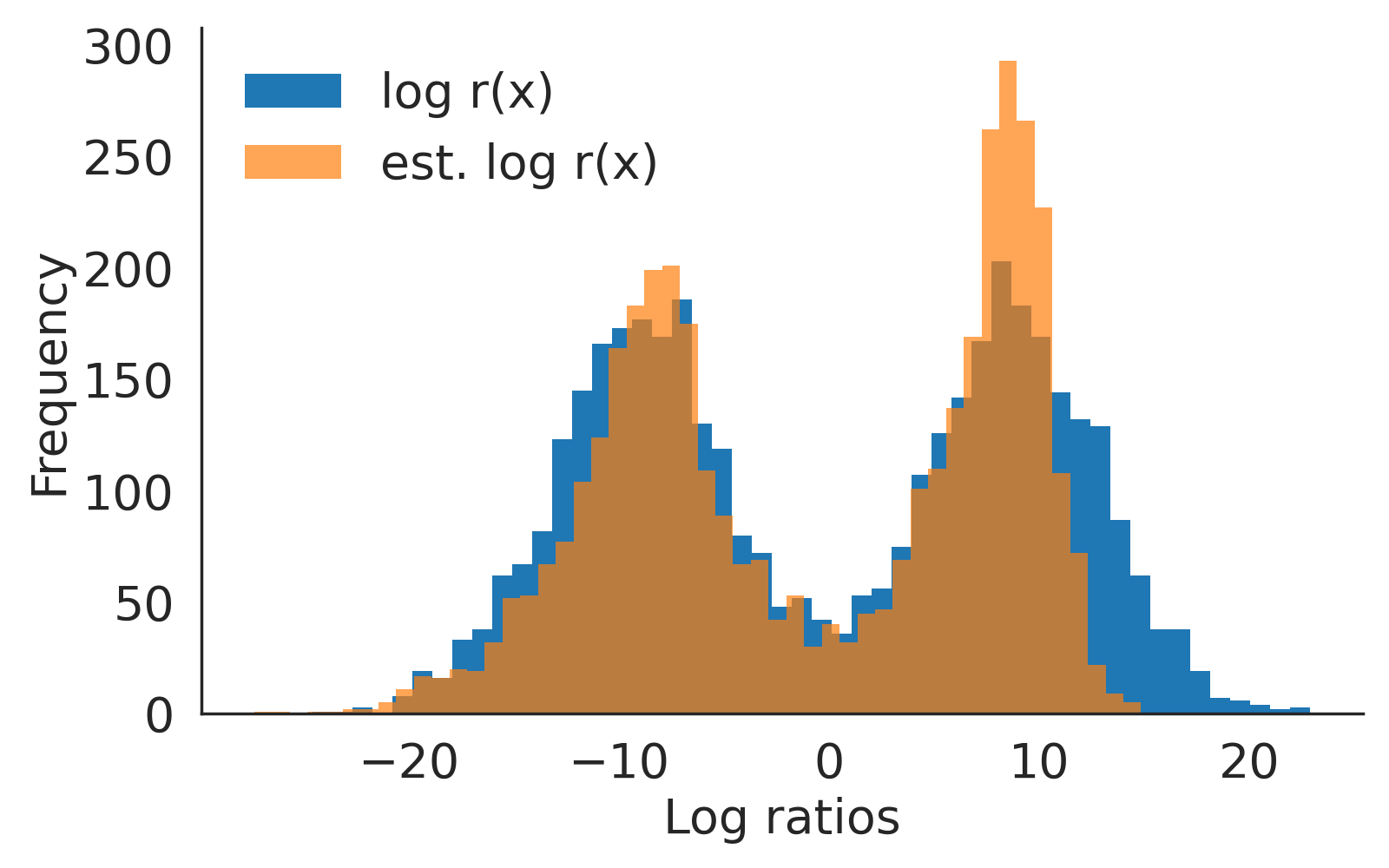}}
        \subfigure[Separate training (flow)]{\includegraphics[width=.25\textwidth]{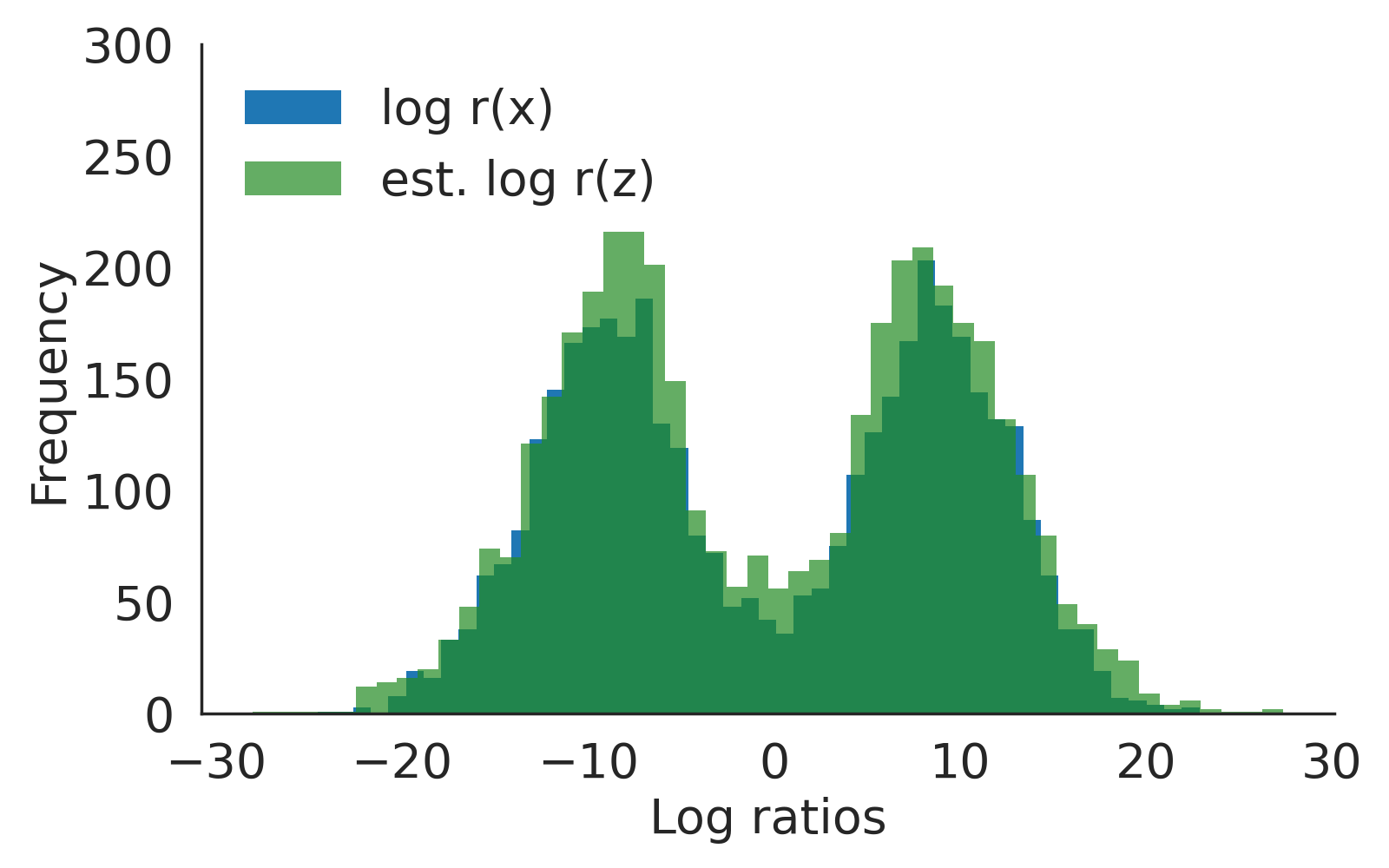}}
        \subfigure[Joint training ($\alpha=0.9$)]{\includegraphics[width=.24\textwidth]{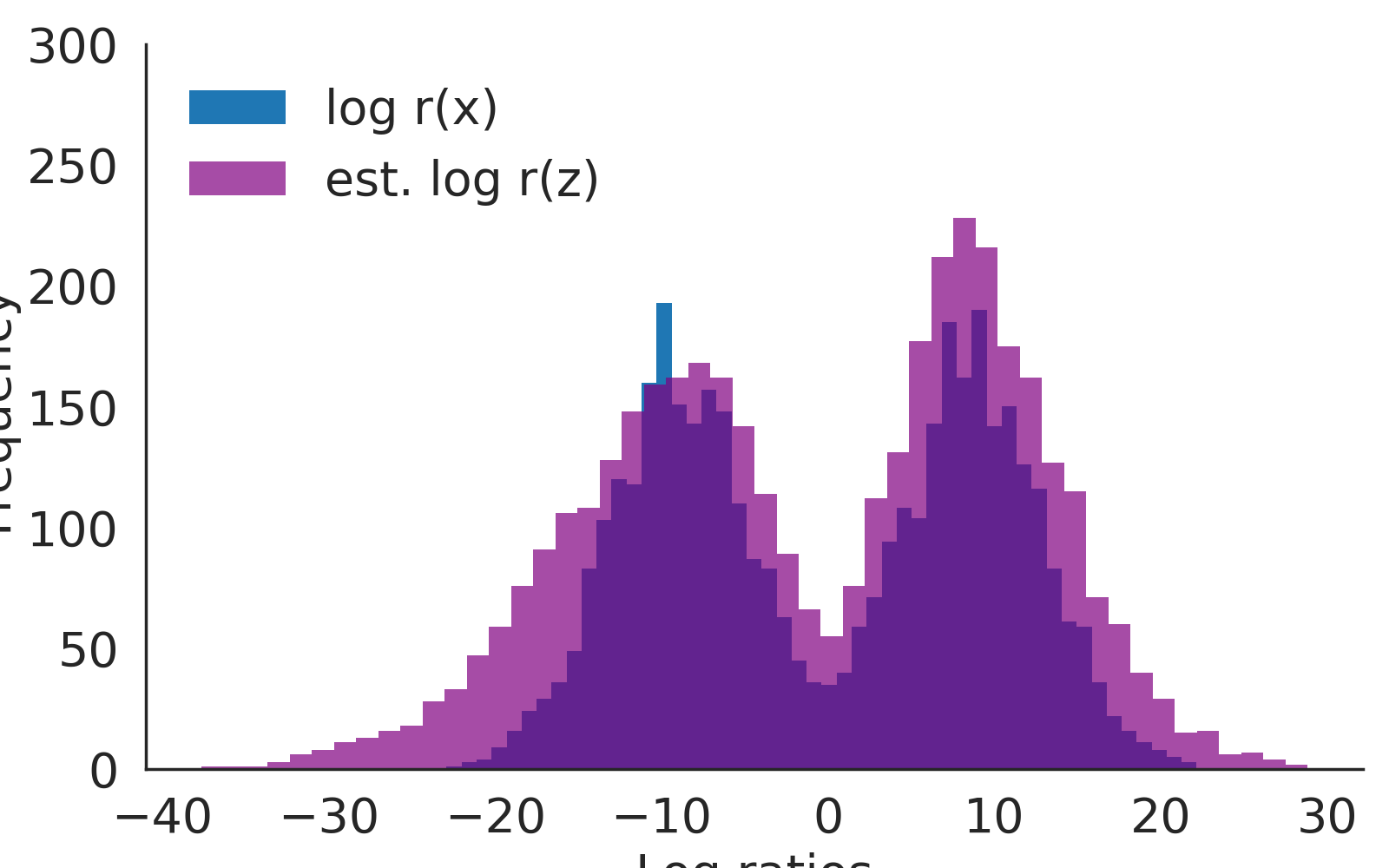}}
        \subfigure[Discriminative training]{\includegraphics[width=.25\textwidth]{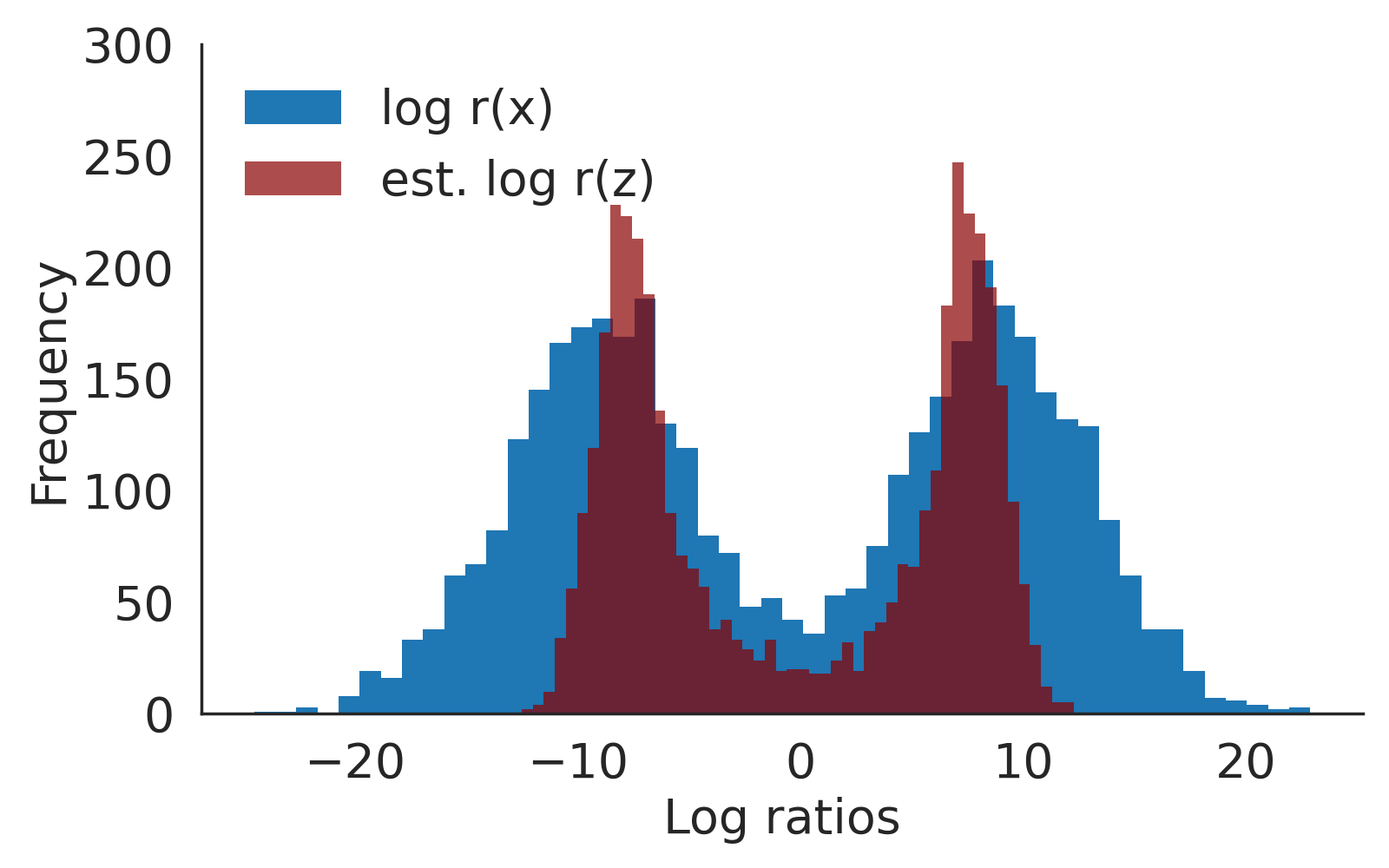}}
        \subfigure[Ground truth data]{\includegraphics[width=.25\textwidth]{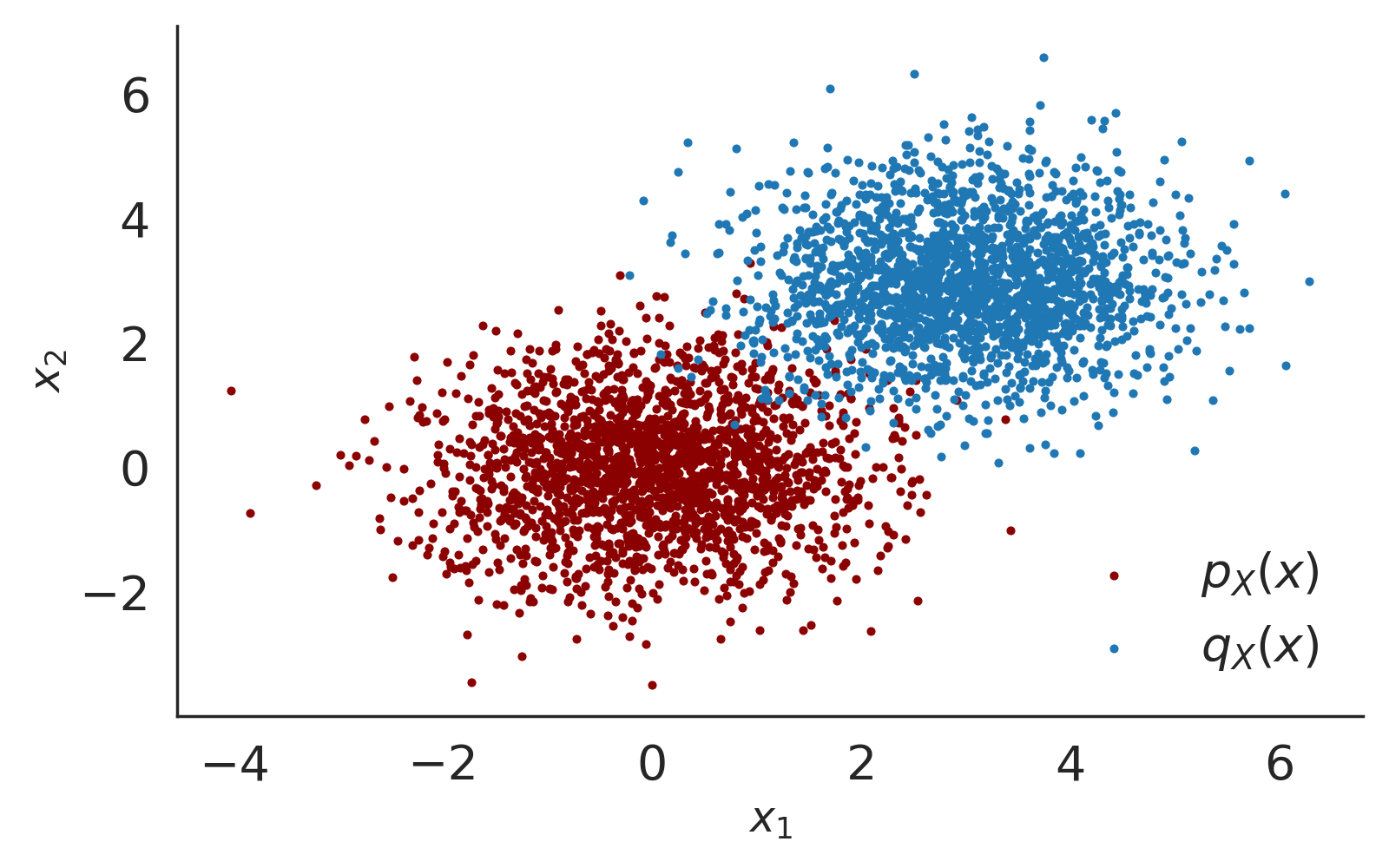}}
        \subfigure[Separate training (flow)]{\includegraphics[width=.25\textwidth]{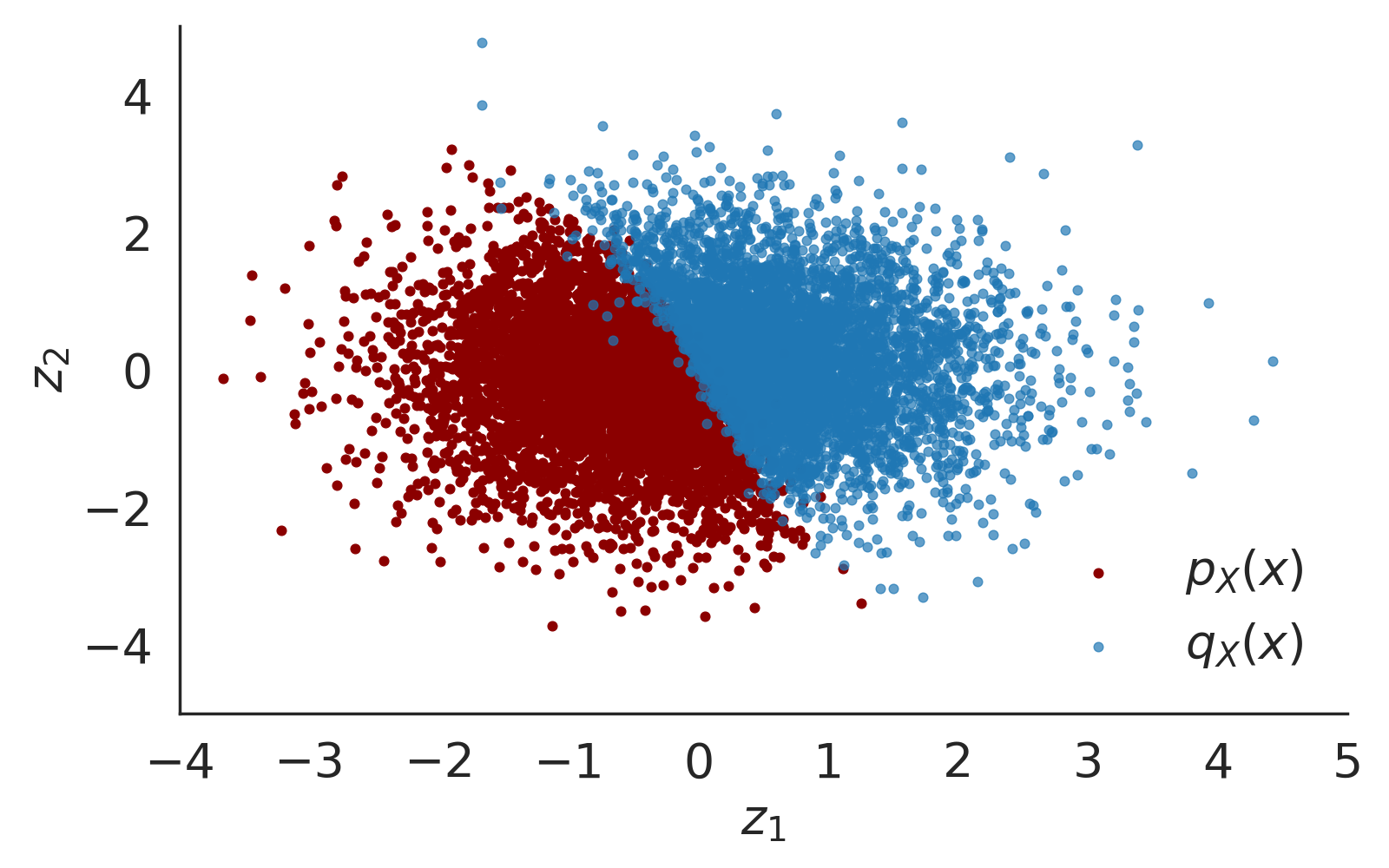}}
        \subfigure[Joint training ($\alpha=0.9$)]{\includegraphics[width=.24\textwidth]{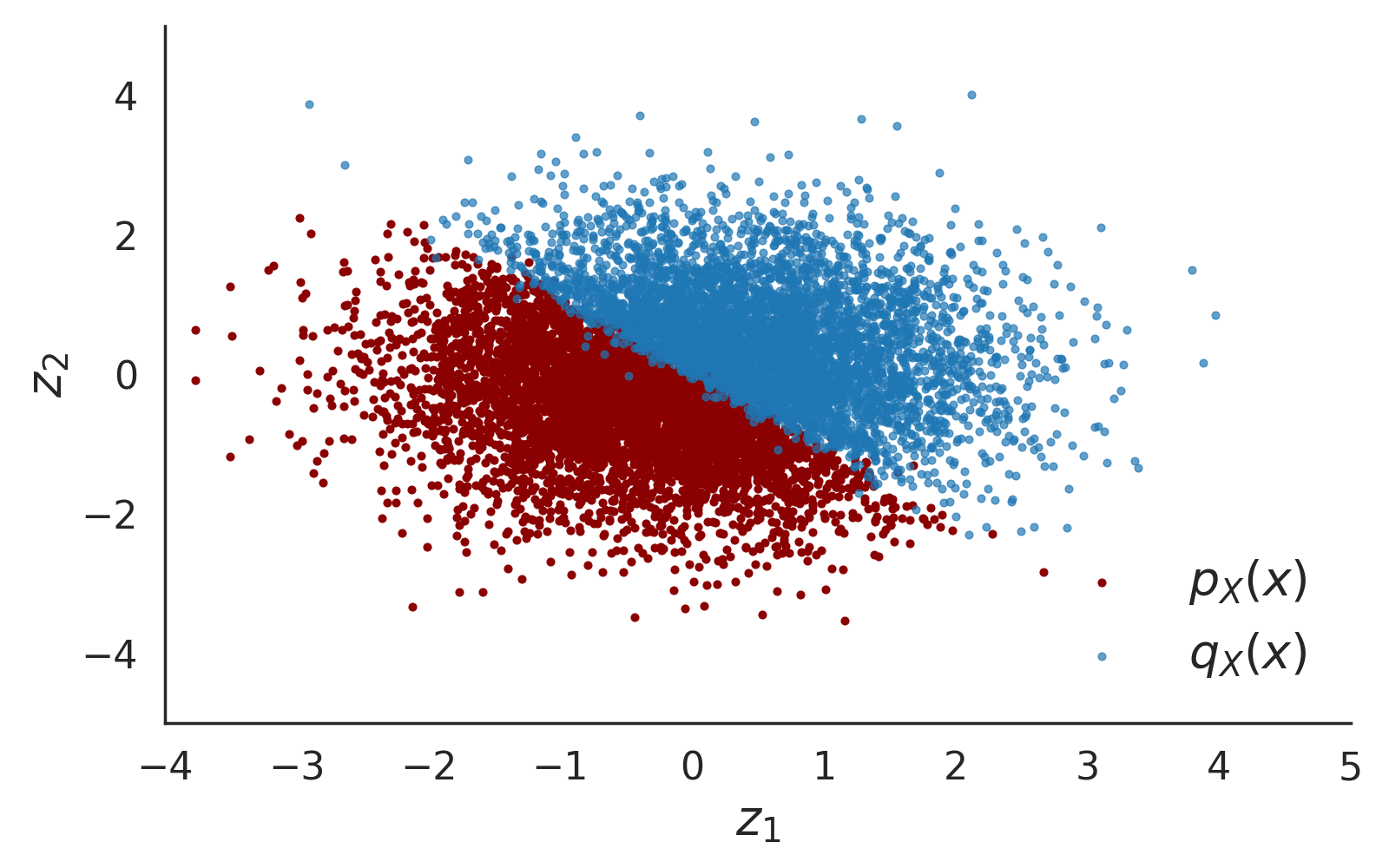}}
        \subfigure[Discriminative training]{\includegraphics[width=.24\textwidth]{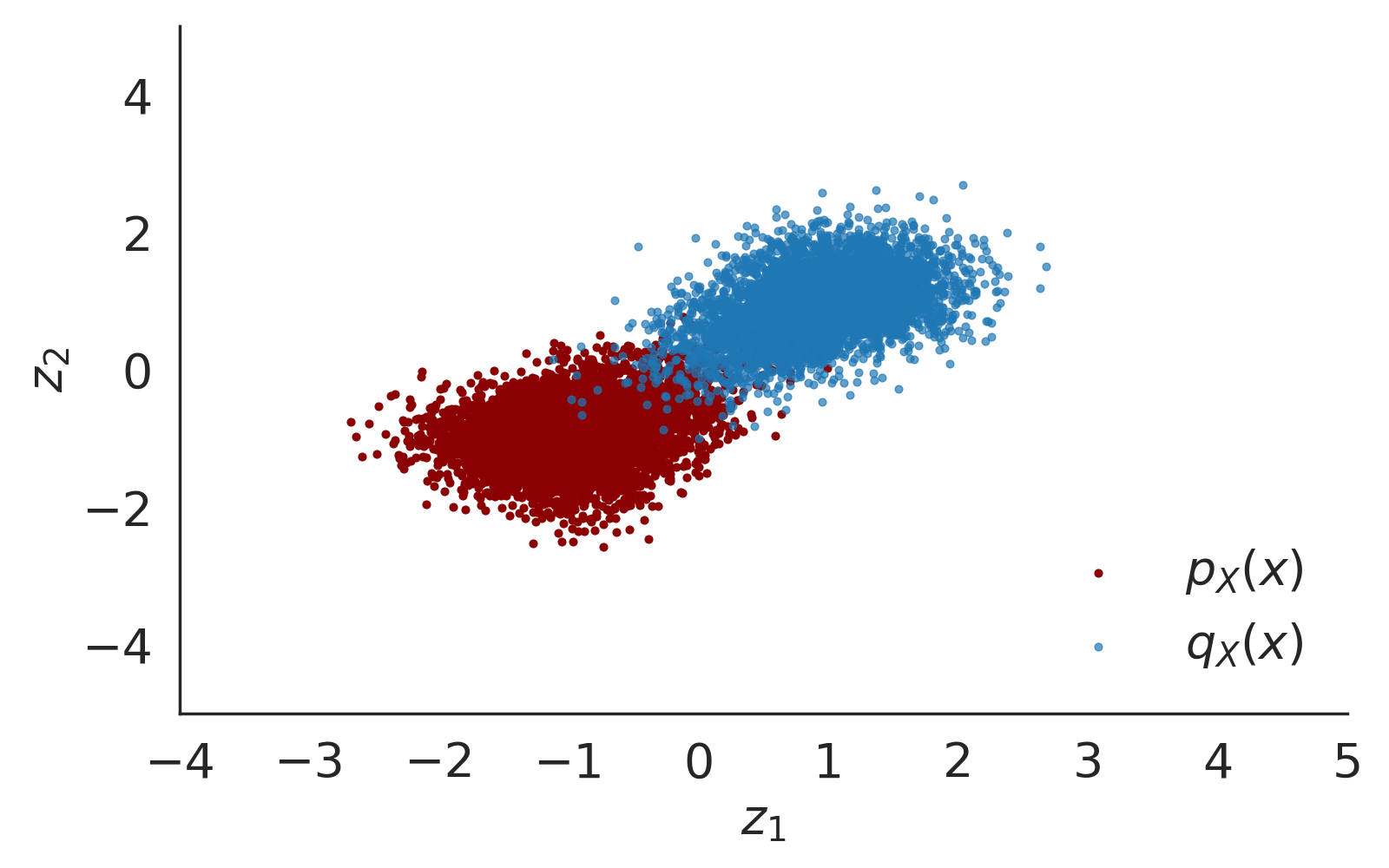}}
    \caption{\textit{Top row:} Motivating example on a synthetic 2-D Gaussian dataset, with learned density ratio estimates by method relative to the ground truth values for (a-d).
    \textit{Bottom row:} Visualizations of the learned encodings for various training strategies for (f-h), with ground truth samples from $p(\bx)$ and $q(\bx)$ in (e). We note that using a pretrained flow as an invertible encoder as in (b) leads to the most accurate density ratio estimates.}
    \label{fig:gmm_encodings}
\end{figure*}
\paragraph{Direct Ratio Estimation.}
From the wealth of alternative estimators for this task \citep{kanamori2009least,sugiyama2012bregman,vapnik2013constructive}, we outline two classical methods which perform density ratio estimation that benefit from featurization as per our framework: (1) Kernel Mean Matching (KMM) \citep{huang2006correcting,gretton2009covariate}, which draws inspiration from moment matching techniques, and (2) the Kullback-Leibler Importance Estimation Procedure (KLIEP) \citep{nguyen2007estimating,sugiyama2008direct}. 

For KMM, density ratio estimates are obtained by projecting all inputs into a 
reproducing kernel Hilbert space (RKHS) $\mathcal{H}$ induced by a characteristic kernel $k: \mathcal{X} \times \mathcal{X} \rightarrow \mathbb{R}$. Although several choices for the kernel are possible, \citeauthor{huang2006correcting} 
use the Gaussian kernel $k(\bx, \bx')=\exp(||\bx-\bx'||^2)$ to arrive at the following objective:
\begin{align*}
    \min_{\hat{r} \in \mathcal{H}}||\mathbb{E}_{q(\bx)}\left[k(\bx,\cdot)\hat{r}(\bx) \right] - \mathbb{E}_{p(\bx)}\left[k(\bx,\cdot)\right]||^2_{\mathcal{H}}
\end{align*}
where both expectations are approximated via Monte Carlo. Intuitively, KMM attempts to match the mean embedding of the two distributions (where the embedding is produced by the canonical feature map defined by $k$) in $\mathcal{H}$. 

For KLIEP, the goal is to estimate density ratios $r(\bx) = p(\bx)/q(\bx)$ such that the Kullback-Leibler (KL) divergence between $p(\bx)$ and $\hat{p}(\bx) = \hat{r}(\bx)q(\bx)$ is minimized:
\begin{align*}
    & \min_{\hat{r}(\bx)} \mathbb{E}_{p(\bx)} \left[ \log \frac{p(\bx)}{\hat{r}(\bx)q(\bx)} \right]\\
    & \textrm{s.t. } \int \hat{r}(\bx) q(\bx) d\bx = 1
\end{align*}
The solution to this constrained optimization problem can also be obtained in $\mathcal{H}$ by parameterizing $\hat{r}_\theta(\bx) = \sum_{i=1}^{n_p} \theta_i k(\bx, \bx_i^p)$ for $\bx^p \sim p(\bx)$ and some kernel $k$, similar in spirit to KMM. 

\paragraph{Probabilistic Classification.} 
Another technique to obtain density ratio estimates is via probabilistic classification, in which a binary classifier $c_\phi: \mathcal{X} \rightarrow [0,1]$ is trained to discriminate between samples from the two densities $p(\bx)$ and $q(\bx)$ \citep{friedman2001elements,gutmann2010noise,sugiyama2012density}. Concretely, suppose we construct a dataset such that all samples $\mathcal{D}_p$ are given the pseudolabel $y=1$, and those from $\mathcal{D}_q$ are labeled as $y=0$. Assuming that the two datasets are equal in size ($n_p = n_q$, though this can be relaxed with a scaling factor), we can use Bayes' rule to arrive at the following expression for the density ratio:
\begin{align*}
    r(\bx) = \frac{p(\bx)}{q(\bx)} = \frac{p(\bx|y=1)}{q(\bx|y=0)} =  \frac{c^*_\phi(\bx)}{1-c_\phi^*(\bx)}
\end{align*}
where $c^*_\phi(\bx) = P(y=1|\bx)$ denotes the Bayes optimal classifier for this particular task.



\section{Featurized Density Ratio Estimation}
\label{method}
\subsection{Motivating Example and Intuition}\label{intuition_ex}
Despite the suite of existing techniques for density ratio estimation, they are of limited use
in settings where $p(\bx)$ and $q(\bx)$ are mismatched in support \citep{cortes2010learning,yamada2013relative,you2019towards,rhodes2020telescoping}
We highlight an illustrative failure case in Figure~\ref{fig:gmm_encodings} on a 2-dimensional toy dataset, where $p(\bx) \sim \mathcal{N}([0, 0]^T,I)$ and $q(\bx) \sim \mathcal{N}([3, 3]^T,I)$. As shown in Figure~\ref{fig:gmm_encodings}(e), the two random variables have regions of minimal overlap -- when training a binary classifier $c_\phi$ to distinguish the two sets of samples, the log-ratio estimates $\log \hat{r}_\phi(\bx)$ learned by the classifier are noticeably inaccurate (Figure~\ref{fig:gmm_encodings}(a)). 

To develop a solution, we consider a simple example to build intuition about the problem (with more details in Appendix~\ref{intuition}). Suppose we want to estimate the density ratios between two 1-dimensional Gaussians, $p \sim \mathcal{N}(m,1)$ and $q \sim \mathcal{N}(-m,1)$, with a finite number of samples $\Dc_p =\{x_i^p\}_{i=1}^n$ and $\Dc_q =\{x_i^q\}_{i=1}^n$ of size $n$ from each. The analytic solution for $r(x) = p(x)/q(x) = \exp \left(-\frac{(x-m)^2 - (x+m)^2}{2} \right) = \exp(2m x)$ for $x\in \mathbb{R}$, which grows exponentially with $m$. 
Without access to the parametric forms of $p$ and $q$, we train a logistic regression model $c_\phi$ to discriminate between $\Dc_p$ and $\Dc_q$, where the maximum likelihood objective is:
\[
\max_{w_0, w_1} \mathbb{E}_p[\log\sigma(w_0 + w_1 \cdot x)] + \mathbb{E}_q[\log  \sigma(-w_0 - w_1 \cdot x)]
\]
where $\sigma(z) = 1/(1 + \exp(-z))$.
Although the logistic regression model is well-specified in this setting, and can achieve Bayes optimal risk in the limit of infinite data, we illustrate what can go wrong in the finite sample regime.

Suppose that $m > 0$ is large -- there exists a large separation between $p$ and $q$. Then, most samples $\Dc_p \sim p$ will take on positive values, and most samples $\Dc_q \sim q$ will be negative. In this situation, the model will be incentivized to push $w_1 \rightarrow \infty$ to maximize the objective. This will lead to wildly inaccurate density ratio estimates, as we know that the true values of $w_1 = 2m$ and $w_0 = 0$ are far from infinity (in fact, $r(x) = \exp(w_1 \cdot x))$.
Thus we must see samples \textit{between} $p$ and $q$ during training: concretely, samples from $p$ such that $x_p \leq 0$ and samples from $q$ such that $x_q \geq 0$. 
But with $n$ samples from $p$, the probability that $\Dc_p$ contains all positive samples is 
$\prod_{i=1}^n P(X_p > 0) \geq (1 - n \exp(-m^2/2))$,
which means that the number of samples required to avoid pathological solutions is exponential in $m^2$.
This implies that density ratio estimation via probabilistic classification in input space is near impossible in such scenarios without extremely large amounts of training data.

\subsection{Methodology}\label{method}
The motivating example in the previous section suggests that it is critical to bring $p$ and $q$ ``closer together" to make the density ratio estimation problem tractable. Our solution is to do so in \textit{latent space} by leveraging an invertible transformation. Concretely, we consider training an invertible deep generative model $f_\theta$ on a \textit{mixture} of $p(\bx)$ and $q(\bx)$, such that $f_\theta(X_p)$ and $f_\theta(X_q)$ are mapped to a common feature space $\mathcal{Z}$. 
The result of utilizing $f_\theta$ as invertible feature map can be visualized in Figure~\ref{fig:gmm_encodings}(f): the flow compresses all data points to lie in different regions of a unit Gaussian ball.
By mapping regions of low density in $\mathcal{X}$ into regions of higher density in $\mathcal{Z}$, and training our probabilistic classifier $c_\phi$ on $f_\theta(\bx) \in \mathcal{Z}$ rather than $\bx \in \mathcal{X}$ directly, this contraction leads to learning more accurate density ratios as shown in Figure~\ref{fig:gmm_encodings}(b). We refer the reader to Appendix~\ref{addtl-results} for additional experimental details and results. 


Our first observation is that as a direct consequence of the invertibility of $f_\theta$, the density ratios obtained in feature space are equivalent to those obtained in input space. We formalize this statement in Lemma~\ref{lemma:cvf} below.
\begin{lemma}\label{lemma:cvf}
Let $X_p \sim p$ be a random variable with density $p$, and $X_q \sim q$ be a random variable with density $q$. Let $f_\theta$ be any invertible mapping. Let $p',q'$ be the densities of $Z_p = f_\theta(X_p)$ and $Z_q = f_\theta(X_q)$ respectively.
Then for any $\bx$:
\[
\frac{p(\bx)}{q(\bx)} =  \frac{p'(f_\theta(\bx))}{q'(f_\theta(\bx))}
\]
\end{lemma}
\begin{proof}
We provide the proof in Appendix~\ref{lemma-proof}.
\end{proof}
This simple observation is quite powerful, as it lends us a general-purpose algorithm that may improve many existing ratio estimation techniques as a black-box wrapper. We provide the pseudocode for our training procedure in Algorithm~\ref{alg:fdre}. Given the two sets of samples, the ratio estimation method, and an invertible generative model family, we first train the normalizing flow on a mixture of the two datasets (Lines 2-3). We then use the trained flow to encode the samples into a common feature space and plug them into the base density ratio estimator algorithm $\texttt{DRE}(\cdot)$ to obtain our featurized density ratio estimator $\hat{r} \circ f_\theta^*$, which is implicitly composed with the trained normalizing flow (Line 6). This algorithm allows us to lightly modify existing approaches such as KMM and KLIEP as detailed in Appendix~\ref{dre_details}, and we explore their featurized variants in our experiments.
\begin{algorithm}[!ht]
  \caption{Featurized Density Ratio Estimation}
  \label{alg:fdre}
  \textbf{Input:} Datasets $\mathcal{D}_p$ and $\mathcal{D}_q$, Density Ratio Estimation Algorithm $\texttt{DRE}$, Invertible Generative Model Family $\{f_\theta, \theta \in \Theta\}$\\
  \textbf{Output:} Featurized Density Ratio Estimator $\hat{r} \circ f_\theta^*$\\
\begin{algorithmic}[1]
\STATE $\triangleright$ Phase 1: Train invertible generative model
  \STATE Concatenate datasets $\mathcal{D} = \{\mathcal{D}_p, \mathcal{D}_q\}$
  \STATE \label{line:flow}Train $f_{\theta^*}$ on $\mathcal{D}$ via maximum likelihood
 \STATE \label{line:test} $\triangleright$ Phase 2: Obtain density ratio estimator
 \STATE $\hat{r} = \texttt{DRE}(f_{\theta^*}(\Dc_p), f_{\theta^*}(\Dc_q))$
 \STATE {\bfseries return} $\hat{r} \circ f_\theta^*$
\end{algorithmic}
\end{algorithm}

\subsection{Training Procedure}
\label{train_methods}
In practice, there are a variety of ways to implement the training procedure as outlined in Algorithm~\ref{alg:fdre}. The most general is \textbf{separate training}, which leverages a pre-trained flow $f_\theta$ as an invertible encoder to map the inputs into a common feature space, prior to ratio estimation. This approach is capable of handling all parametric and non-parametric techniques which operate directly in input space.

In the probabilistic classification setting, where the density ratio estimation algorithm $\texttt{DRE}(\cdot)$ requires learning a binary classifier $c_\phi$ to distinguish between $\Dc_p$ and $\Dc_q$, we can \textit{adapt} the normalizing flow $f_\theta$ to account for the known structure of $c_\phi$. We call this procedure \textbf{joint training}. Both the normalizing flow $f_\theta$ and the discriminative classifier $c_\phi$ are trained jointly via the following objective:
\begin{equation} \label{joint_obj}
    \mathcal{L}_{\textrm{joint}}(\theta,\phi) = \alpha \mathcal{L}_{\textrm{sup}}(\theta,\phi) + (1-\alpha) \mathcal{L}_{\textrm{flow}}(\theta)
\end{equation}
where $\mathcal{L}_{\textrm{sup}}$ denotes the standard binary cross entropy (logistic) loss, $\mathcal{L}_{\textrm{flow}}$ denotes the maximum likelihood objective for the flow $f_\theta$, and $\alpha \in [0,1]$ is a hyperparameter which balances the importance of the two terms in the loss function. This approach is quite common in learning deep hybrid models \citep{kuleshov2017deep,nalisnick2019hybrid}. 

 Finally, we explore \textbf{discriminative training}, where we modify the classifier $c_\phi$'s architecture to incorporate that of the flow $f_\theta$
to build an ``invertible" classifier $c_{\phi,\theta}: \mathcal{X} \rightarrow [0,1]$ that is trained solely via the logistic loss $\mathcal{L}_\textrm{sup}(\theta,\phi)$. This is inspired by the strong performance of invertible networks such as i-RevNet \citep{jacobsen2018revnet}, i-ResNet \citep{behrmann2019invertible}, and Mintnet \citep{song2019mintnet} on downstream classification tasks.

\subsection{Characterization of the Learned Feature Space}
At a first glance, Lemma~\ref{lemma:cvf} appears to suggest that \textit{any} feature space induced by an invertible map should work well for density ratio estimation, as long as $f_\theta(X_p)$ and $f_\theta(X_q)$ are closer together than $X_p$ and $X_q$.
To gain further insight into the desirable characteristics of the learned feature space, we visualize the encodings of the various training strategies in Figure~\ref{fig:gmm_encodings}. 
For both a pretrained (Figure~\ref{fig:gmm_encodings}(f)) and jointly trained (Figure~\ref{fig:gmm_encodings}(g)) normalizing flow, the data points are mapped to lie closer together in different regions of the unit Gaussian ball. However, for the discriminatively trained classifier equipped with an invertible ``encoder" (Figure~\ref{fig:gmm_encodings}(h)), the encoded examples more closely resemble the shape of the original inputs (Figure~\ref{fig:gmm_encodings}(e)).
This observation, combined with the low quality density ratio estimates in Figure~\ref{fig:gmm_encodings}(d) relative to the other training methods (Figure~\ref{fig:gmm_encodings}(b-c)), suggests that maximum likelihood training of the normalizing flow $f_\theta$ \textit{in addition to} shrinking the gap between the densities $p$ and $q$ is crucial for obtaining accurate density ratios in feature space. We hypothesize that mapping the observations into a unit Gaussian ball is an important property of our method, and we save an in-depth theoretical analysis of this phenomenon for future work. 
\section{Theoretical Analysis}
\label{theory}
In this section, we provide theoretical justifications for several properties of the featurized density ratio estimator. As a consequence of Lemma~\ref{lemma:cvf}, we find that our estimator inherits many of the desirable properties of the original estimator.

\subsection{Properties of the Estimator}
\paragraph{Unbiasedness.} Unbiasedness is one of the most fundamental desiderata of a statistical estimator, as it guarantees that the estimated parameter is equivalent to the parameter's true value in expectation. In Corollary~\ref{lemma:cvf:unbiased}, we prove that unbiasedness of the featurized ratio estimator follows directly if the original estimator is also unbiased. 
\begin{corollary}\label{lemma:cvf:unbiased}
Let $\Dc_p$ be $n$ i.i.d samples from density $p$, and $\Dc_q$ be $n$ i.i.d samples from density $q$.
Let $\hat{r}(\bx)$ obtained from $\hat{r} = \dre \left(\Dc_p,\Dc_q \right)$
be an unbiased estimator of $r(\bx) = \frac{p(\bx)}{q(\bx)}$ and any $p,q$, and let $f_\theta$ denote any invertible mapping. 
Then,
$(\hat{r}' \circ f_\theta)(\bx)$ obtained from $\hat{r}' = \dre \left(f_\theta(\Dc_p), f_\theta(\Dc_q) \right)$
is also an unbiased estimator of $\frac{p(\bx)}{q(\bx)}$ for any $p,q$.
\end{corollary}
\begin{proof}
We provide the proof in Appendix~\ref{unbiased-supp}.
\end{proof}

\paragraph{Consistency.} Consistency is another key property in a statistical estimator, as it guarantees that in the limit of infinite data used in the estimation procedure, the probability that the estimator becomes arbitrarily close to the true parameter converges to one. 
We prove in Corollary~\ref{lemma:cvf:consistent} that consistency of the featurized density ratio estimator also follows if the original density ratio estimator is consistent. This is desirable, as estimators such as the KLIEP and KMM (with universal kernels) are both consistent \citep{huang2006correcting,gretton2009covariate,sugiyama2012density}. 

\begin{corollary}\label{lemma:cvf:consistent}
Let $\Dc_p$ be $n$ i.i.d samples from density $p$, and $\Dc_q$ be $n$ i.i.d samples from density $q$. Let $\hat{r}(\bx)$ obtained from $\hat{r} = \dre(\Dc_p,\Dc_q)$
be a consistent estimator of $\frac{p(\bx)}{q(\bx)}$ for all $\bx \in \mathcal{X}$ and for any $p,q$. Let $f_\theta$ be any invertible mapping. Then, $(\hat{r}' \circ f_\theta)(\bx)$ obtained from $\hat{r}' = \dre \left(f_\theta(\Dc_p), f_\theta(\Dc_q)\right)$
is also a consistent estimator of $\frac{p(\bx)}{q(\bx)}$ for any $p,q$.
\end{corollary}
\begin{proof}
We provide the proof in Appendix~\ref{consistency-supp}.
\end{proof}





\section{Experimental Results}
\label{experiments}
In this section, we are interested in empirically investigating the following questions:
\begin{enumerate}
    \item Are the density ratios learned in feature space indeed more accurate than those learned in input space?
    \item Do estimates in feature space yield better performance on downstream tasks that rely on  density ratios?
\end{enumerate}
For conciseness, we report the average over several runs for all experiments and report complete results in Appendix~\ref{addtl-results}. 

\textbf{Datasets.}
We evaluate the efficacy of featurized density ratio estimation on both synthetic and real-world datasets. The synthetic experiments include toy examples on Gaussian mixtures of varying dimensionality (see Appendix~\ref{kernel_synth}), as well as datasets from the UCI Machine Learning Repository \citep{Dua:2019}. For more challenging scenarios, we consider MNIST \citep{lecun1998mnist} and Omniglot \citep{lake2015human}. Additional details on the dataset construction for all experiments can be found in Appendix~\ref{exp_details}.

\textbf{Models.}
We train different classifiers depending on the difficulty of the classification task, but largely keep the same architecture (either an MLP or CNN) across different tasks. 
For the normalizing flow, we utilize the Masked Autoregressive Flow (MAF) for all datasets \citep{papamakarios2017masked}. We train the MAF separately on the mixture of the two datasets prior to density ratio estimation for all experiments with the exception of the MI estimation experiment in Section~\ref{mi_exp}, where we explore various training strategies mentioned in Section~\ref{train_methods}.
For additional details regarding architecture design and relevant hyperparameters, we refer the reader to Appendix~\ref{appendix_arch}. 

\subsection{Domain Adaptation} \label{synth_exp}
We first pair our method with two existing techniques, KMM and KLIEP, to assess whether estimating ratios in feature space improves performance on domain adaptation tasks with: 1) 2-D Gaussian mixtures and 2) the UCI Breast Cancer dataset. On the synthetic dataset, our method achieves a lower test error than both baseline logistic regression (without importance weighting) and reweighted logistic regression using density ratios estimated by KMM and KLIEP in input space. See Appendix~\ref{kernel_synth} for full results.

The UCI Breast Cancer dataset consists of $n=699$ examples from 2 classes: benign ($y=1$) and malignant ($y=-1$), where each sample is a vector of $9$ features. We replicate the experimental setup of \citep{huang2006correcting} to construct a source dataset with a heavily downsampled number of benign labels, while leaving the target dataset as is. 
After learning the importance weights via density ratio estimation on a mixture of the source and (unlabeled) target datasets, 
we train a support vector machine (SVM) with a Gaussian kernel of bandwidth $\sigma = 0.1$ and varying penalty hyperparameter values $C = \{0.1, 1, 10, 100\}$ with importance weighting on the source domain. The binary classifier is then tested on the target domain. As shown in Figure~\ref{fig:kmm_uci}, when applied to KMM, for nearly all values of $C$, our method (\texttt{z-dre}) achieves the lowest test error on the target dataset compared to both a vanilla SVM (\texttt{baseline}) and a reweighted SVM with density ratio estimates computed in input space (\texttt{x-dre}). Additionally, we note that our method achieves the absolute lowest test error across varying values of $C$. We report the average values of our KMM experiments over $30$ runs in Figure~\ref{fig:kmm_uci}.

All methods performed poorly overall for our KLIEP experiments. This result aligns with many past works with KLIEP importance-weighted classification; empirically, KLIEP only outperforms baseline unweighted classifiers on synthetic data, while on more complex datasets (e.g. UCI), KLIEP shows no significant improvements \citep{sugiyama2007direct, tsuboi2008direct, yamada2009direct, loog2012nearest}. In order to confirm the consistency of this behavior, we performed an additional experiment with a slightly different dataset-biasing process in which data points that were further from the mean were selected less often, similarly to \cite{huang2006correcting}; we report more details on the biased subsampling process in Appendix~\ref{appendix:kmm-kliep-datasets}. We used two datasets: 1) the UCI Blood Transfusion dataset and 2) the UCI Wine Quality dataset and found that both reweighted classifiers performed similarly to the baseline. Notably, our \texttt{z-dre} method does not degrade the performance of KLIEP. Table~\ref{table:kliep_exp} shows our results.
\begin{table*}[ht]
\centering
\resizebox{0.87\linewidth}{!}{
\begin{tabular}{l|l|l|l|l}
\toprule
\textbf{Blood Transfusion} & $C=0.1$ & $C=1$ & $C=10$ & $C=100$ \\
\midrule
KLIEP with DRE in z-space (ours) & $0.235 \pm .0274$ & $0.235 \pm .0274$ & $0.234 \pm .0283$ & $0.234 \pm .0282$ \\ 
KLIEP with DRE in x-space & $0.235 \pm .0274$ & $0.235 \pm .0274$ & $0.234 \pm .0282$ & $0.233 \pm .0284$ \\
Unweighted SVM baseline & $0.235 \pm .0274$ & $0.235 \pm .0274$ & $0.234 \pm .0287$ & $0.234 \pm .0285$  \\
\midrule
\textbf{Wine Quality} & $C=0.1$ & $C=1$ & $C=10$ & $C=100$ \\
\midrule
KLIEP with DRE in z-space (ours) & $0.304 \pm .0120$ & $0.260 \pm .00937$ & $0.265 \pm .00817$ & $0.290 \pm .00987$ \\ 
KLIEP with DRE in x-space & $0.304 \pm .0123$ & $0.262 \pm .0105$ & $0.266 \pm .0113$ & $0.290 \pm .0103$  \\
Unweighted SVM baseline & $0.302 \pm .0274$ & $0.257 \pm .0074$ & $0.262 \pm .00863$ & $0.289 \pm .0933$  \\
\bottomrule
\end{tabular}
}
\caption{KLIEP test error of each method on binary classification for the UCI Blood Transfusion and Wine Quality datasets. Results are averaged over 30 runs. KLIEP reweighting in general does not offer significant improvement over the unweighted baseline--in particular, our method (z-space) doesn't degrade performance.}
\label{table:kliep_exp}
\end{table*}
\begin{figure}[!ht]
    \centering
        \includegraphics[width=.4\textwidth]{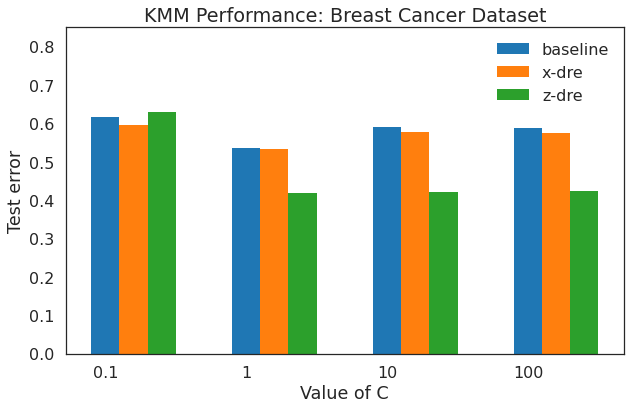}
    \caption{KMM test error on binary classification of the UCI Breast Cancer dataset using a SVM with varying $C$. Lower is better. Results are averaged over 30 runs.}
    \label{fig:kmm_uci}
   \end{figure}



\subsection{Mutual Information Estimation}\label{mi_exp}
Next, we test our approach on a mutual information (MI) estimation task between two correlated 20-dimensional Gaussian random variables, where the ground truth MI is tractable. 
MI estimation between two random variables $X_p$ and $X_q$ is a direct application of density ratio estimation, as the problem can be reduced to estimating average density ratios between their joint density and the product of their marginals. If we let $v$ denote the joint density of $X_p$ and $X_q$, we can see that: $I(X_p;X_q) = \mathbb{E}_{v(\bx^p,\bx^q)}\left[\log \frac{v(\bx^p,\bx^q)}{p(\bx^p)q(\bx^q)}\right]$.
We adapt the experimental setup of \citep{belghazi2018mutual,poole2019variational,song2019understanding} to use a correlation coefficient of $\rho=0.9$. 

We further explore the effect of the various training strategies as outlined in Section~\ref{train_methods}. While we use a MAF as the normalizing flow for all configurations, we evaluate our approach against: (a) the baseline classifier (\texttt{baseline}); (b) the two-stage approach (\texttt{separate}), where the flow is trained first on a mixture of $\mathcal{D}_p$ and $\mathcal{D}_q$ before training the classifier on the encoded data points; (c) jointly training the flow and the classifier (\texttt{joint}); and (d) a purely discriminative approach where the classifier architecture has a flow component (\texttt{disc-only}). For joint training, we sweep over $\alpha=\{0.1, 0.5, 0.9\}$. As shown in Figure~\ref{fig:mi}, the probabilistic classifier trained in feature space (after encoding the data using the normalizing flow) via our method outperforms relevant baselines. Interestingly, we find that for the joint training, higher values of $\alpha$ (which places a greater emphasis on the classification loss $\mathcal{L}_\text{sup}$ rather than $\mathcal{L}_{\text{flow}}$ as in Eq.~\ref{joint_obj}) leads to more accurate MI estimates. For additional details on the data generation process and experimental setup, we refer the reader to Appendix~\ref{appendix_arch}. 

\begin{figure}[ht!]
    \centering
        \includegraphics[width=.47\textwidth]{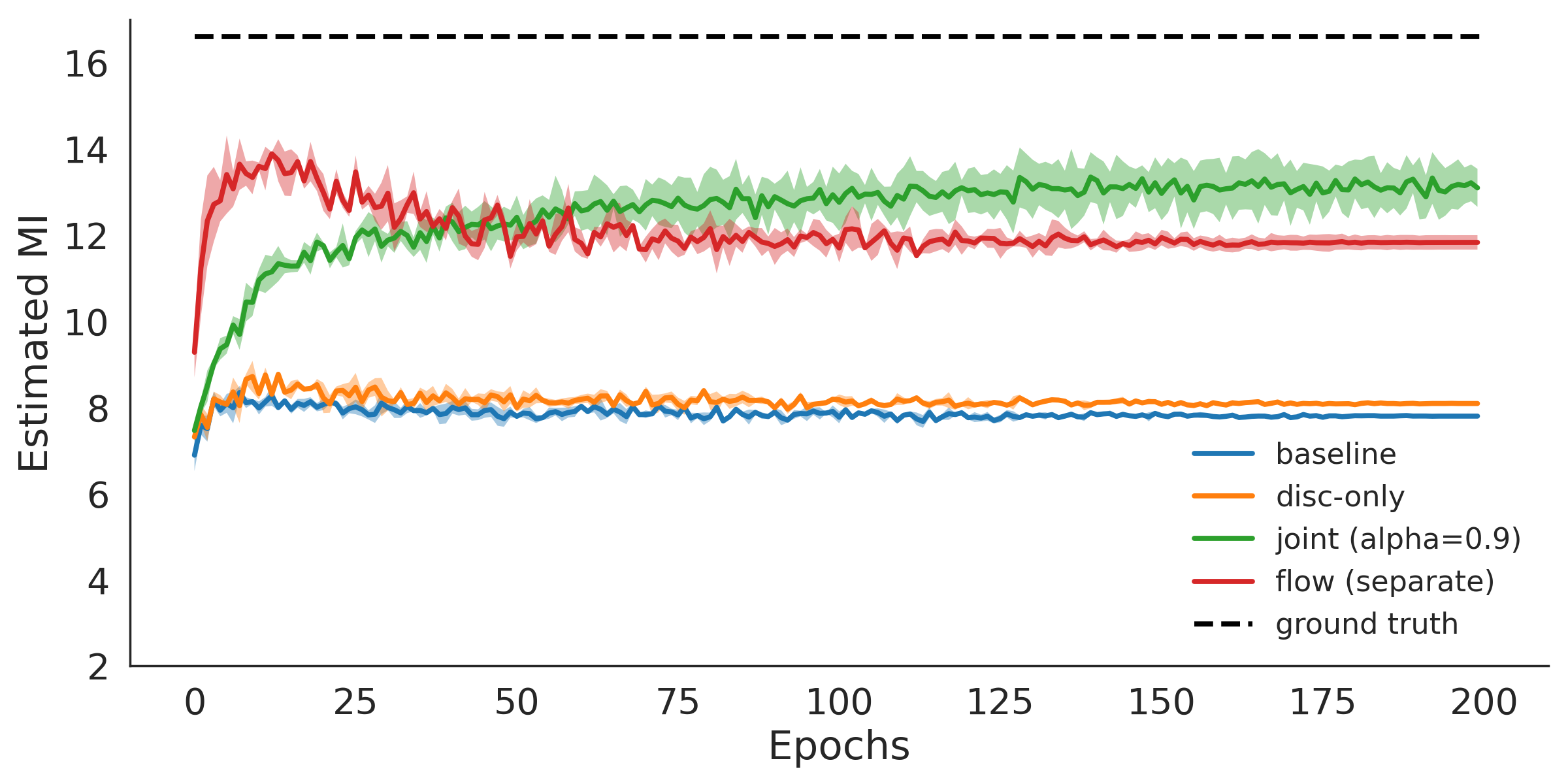}
    \caption{Estimated MI for various training strategies. The true MI for the corresponding value of $\rho=0.9$ is $16.67$. While separate training outperforms all baselines, joint training achieves competitive performance with larger $\alpha$.}
    \label{fig:mi}
   \end{figure}

\subsection{Targeted Generation with MNIST} \label{mnist_exp}
For this experiment, we evaluate the effectiveness of our learned density ratio estimates on a targeted generation task using the MNIST dataset. Our goal is to generate samples according to a target distribution $q(\bx)$ in a data-efficient manner, given samples from both $p(\bx)$ and $q(\bx)$. We test two scenarios: (a) \texttt{diff-digits}: a subset of MNIST in which $p(\bx)$ is comprised of the digits labeled \{1,2\}, and $q(\bx)$ which is comprised of the digits labeled \{0,7\}; (b) \texttt{diff-background}: a setting in which $p(\bx)$ contains the original MNIST digits (black background, white digits); and $q(\bx)$ contains the same examples but with flipped colors (white background, black digits). The second scenario is trickier than the first, since there exists an obvious gap between the two distributions. We also explore the effect of the target dataset size $q(\bx)$ in learning accurate density ratios. Following the setup of \citep{choi2020fair}, we sweep over various sizes of $q(\bx)$ relative to $p(\bx)$, which we call \texttt{perc}=\{0.1, 0.25, 0.5, 1.0\} (where $\texttt{perc}=0.5$ indicates that $\mathcal{D}_q$ is 50\% the size of $\mathcal{D}_p$). After training a MAF on both $\mathcal{D}_p$ and $\mathcal{D}_q$ and obtaining density ratio estimates (importance weights), we sample from the trained MAF via sampling-importance-resampling (SIR) \citep{liu1998sequential,doucet2000sequential} at generation time.

As shown in Table~\ref{table:mnist_exp}, we achieve greater success in the targeted generation task when performing SIR with importance weights learned in feature space. Averaged across \texttt{perc}=\{0.1, 0.25, 0.5, 1.0\} with $1000$ generated samples each, our method generates 19.1\%  more samples from $q(\bx)$ relative to the pretrained flow and 6.7\% more samples than the baseline with importance weights learned in input space on the \texttt{diff-digits} task. Similarly for the \texttt{diff-background} task, our framework generates 18.8\% more samples from $q(\bx)$ relative to the pretrained flow and 16.4\% more samples than the baseline.
For additional experimental details, as well as the generated samples, 
we refer the reader to Appendix~\ref{exp_details} and \ref{addtl-results}.
\begin{table*}[ht]
\centering
\begin{tabular}{l|l|l|l|l}
\toprule
\textbf{Different Digits} & \texttt{perc}=0.1 & \texttt{perc}=0.25 & \texttt{perc}=0.5 & \texttt{perc}=1.0 \\
\midrule
SIR with IW(z) (ours) & \textbf{0.447} $\pm$ \textbf{0.020} & 0.518 $\pm$ 0.008& \textbf{0.777} $\pm$ \textbf{0.018} & \textbf{0.860} $\pm$ \textbf{0.004} \\ 
SIR with IW(x) & 0.441 $\pm$ 0.002 & \textbf{0.528} $\pm$ \textbf{0.004} & 0.639 $\pm$ 0.007 & 0.754 $\pm$ 0.007  \\
Regular sampling & 0.406 $\pm$ 0.055 & 0.457 $\pm$ 0.07 & 0.596 $\pm$ 0.052 & 0.720 $\pm$ 0.035  \\
\midrule
\textbf{Different Backgrounds} & \texttt{perc}=0.1 & \texttt{perc}=0.25 & \texttt{perc}=0.5 & \texttt{perc}=1.0\\
\midrule
SIR with IW(z) (ours) & \textbf{0.186} $\pm$ \textbf{0.005} & \textbf{0.377} $\pm$ \textbf{0.001} & \textbf{0.580} $\pm$ \textbf{0.005} & \textbf{0.732} $\pm$ \textbf{0.008} \\ 
SIR with IW(x) & 0.085 $\pm$ 0.003 & 0.202 $\pm$ 0.003 & 0.345 $\pm$ 0.013 & 0.528 $\pm$ 0.022  \\
Regular sampling & 0.084 $\pm$ 0.003 & 0.196 $\pm$ 0.003 & 0.304 $\pm$ 0.003 & 0.493 $\pm$ 0.016  \\
\bottomrule
\end{tabular}
\caption{MNIST targeted generation results averaged over 3 runs. Columns show the fraction of generated samples with the target attribute (higher is better) across varying sizes of the target dataset. $1000$ samples were generated for each setup.}
\label{table:mnist_exp}
\end{table*}

\subsection{Classification with Data Augmentation on Omniglot}
Finally, we follow the experimental setup of \citep{grover2019bias} by utilizing Data Augmentation Generative Adversarial Networks (DAGAN) \citep{antoniou2017data} as a generative model for importance-weighted data augmentation on the Omniglot dataset \citep{lake2015human}. Since Omniglot is comprised of 1600+ classes with only 20 examples per class, the goal of this experiment is improve the performance of a downstream multi-class classifier by effectively leveraging additional samples generated by the DAGAN. To do so, we train a separate probabilistic classifier to distinguish between the true and the generated examples, yielding importance weights for each synthetic example that can be used for training the downstream classifier of interest.

We first train a MAF on a mixture of the training examples and generated samples, encode all the data using the flow, and obtain importance weights via the encodings. The importance weights are obtained by training a binary classifier on the featurized inputs. We experiment with different baselines: (a) training the classifier without any data augmentation (\texttt{Data-only}); (b) training the classifier on purely synthetic samples (\texttt{Synthetic-only}); (c) training the classifier with data-augmentation without any importance weighting (\texttt{Mixture-only}); (d) the data-augmented classifier with importance weights obtained from input space (\texttt{Mixture + IW(x)}); and (e) the data-augmented classifier with importance weights obtained from feature space (\texttt{Mixture + IW(z)}). As shown in Table~\ref{table:omniglot_exp}, the importance weights learned in the feature space show a significant boost in overall downstream classification accuracy as compared to relevant baselines: our method improves 3.7\% over the \texttt{Data-only} baseline, and 2.2\% over the highest performing baseline. We refer the reader to Appendix~\ref{addtl-results} for additional experimental details and results.
\begin{table*}[ht]
\centering
\begin{tabular}{l|l|l|l|l|l}
\toprule
\textbf{Dataset} & Data-only & Synthetic-only & Mixture-only & Mixture + IW(x) & Mixture + IW(z)\\
\midrule
\textbf{Accuracy} & $0.756 \pm 0.001$ & $0.557 \pm 0.003$ & $0.767 \pm 0.003$ & $0.765 \pm 0.005$ & $\mathbf{0.784 \pm 0.007}$ \\ 
\bottomrule
\end{tabular}
\caption{Downstream predictive accuracy on the Omniglot dataset. Standard errors are computed over 3 runs.}
\label{table:omniglot_exp}
\end{table*}


\section{Related Work}
\label{related}
\paragraph{Density Ratio Estimation in Feature Space.} Although density ratio estimation in machine learning has an extremely rich history \citep{friedman2001elements,huang2006correcting,nguyen2007estimating,gutmann2010noise,sugiyama2012density}, there is considerably less work exploring the method's counterpart in feature space. 
\citep{rhodes2020telescoping}, while tackling the same problem of density ratio estimation between two different data distributions, adopts a different approach than our framework. In particular, they propose a divide-and-conquer solution by constructing intermediate distributions between the two densities $p(\bx)$ and $q(\bx)$, and requires the training of a multi-task logistic regression model rather than a single binary classifier. Their interpolation technique, which is also conducted in the latent space of a normalizing flow in one of their experiments, is complementary to our framework -- investigating the combination of these two approaches would be interesting future work.
Additionally, density ratio estimation (in the form of learning importance weights) has been popular in a variety of domain adaptation approaches such as \citep{bickel2007discriminative,long2015learning,you2019towards} which leverage a feature extractor to project the inputs into a lower-dimensional manifold prior to estimation. Although our approach shares a similar idea, the invertibility of our feature map guarantees that the density ratios between input space and feature space are equivalent -- this is not necessarily true if the inputs are lossily compressed. 

\paragraph{Neural Hybrid Models.} Combining both generative and discriminative training approaches in neural networks has previously been explored in the literature \citep{maaloe2016auxiliary,gordon2017bayesian,kuleshov2017deep}. Our work bears most similarity to \citep{nalisnick2019hybrid}, as we also require learning an invertible generative model and a discriminator. However, our method does not require that the normalizing flow be trained together with the probabilistic classifier, and can be used for more downstream applications beyond out-of-distribution detection and semi-supervised learning, as our goal is to accurately estimate density ratios. Additionally, our approach is related to conditional normalizing flows such as \citep{dinh2019rad} and \citep{winkler2019learning} which explicitly partition the latent space of the flow $p_Z(\bz)$ to map different components of the input into disjoint regions in the prior. Although we empirically verify that this is also the case for our method, it is more general precisely because the best partitioning is \textit{learned} by the model.


\section{Conclusion}
In this paper, we proposed a general-purpose framework for improved density ratio estimation in settings where the two underlying data distributions of interest are sufficiently different. 
The key component of our approach is a normalizing flow that is trained on a mixture of the data sources, which is then used to encode the data into a shared feature space prior to estimating density ratios. 
By leveraging the invertibility of the flow, we showed that the ratios of the densities in feature space are not only identical to those in input space, but are also easier to learn. Additionally, our method is applicable to a suite of existing density ratio estimation techniques. 
Empirically, we demonstrated the utility of our framework on various combinations of density ratio estimation techniques and downstream tasks that rely on accurate density ratios for good performance, such as domain adaptation, mutual information estimation, and targeted generation in deep generative models.
We provide a reference implementation in PyTorch \citep{paszke2017automatic}, and the codebase for this work is open-sourced at \url{https://github.com/ermongroup/f-dre}.

One limitation of our method is the need to train a normalizing flow on a mixture of the two datasets if a pre-trained model is not available; this may be difficult if the generative model must be extremely high-capacity. For future work, it would be interesting to explore whether the necessity of strict invertibility of the flow can be relaxed, and to gain a deeper theoretical understanding of the role of maximum likelihood training in our framework.


\begin{acknowledgements}
We are thankful to Jiaming Song, Daniel Levy, Rui Shu, Ishaan Gulrajani, and Kuno Kim for insightful discussions and feedback.
KC is supported by the NSF GRFP, Stanford Graduate Fellowship, and Two Sigma Diversity PhD Fellowship.
This research was supported by NSF (\#1651565, \#1522054, \#1733686), ONR (N00014-19-1-2145), AFOSR (FA9550-19-1-0024), ARO (W911NF2110125), and Amazon AWS.
\end{acknowledgements}

\bibliography{main}

\raggedbottom
\pagebreak

\onecolumn
\section*{Appendix}
\renewcommand{\thesubsection}{\Alph{subsection}}

\newtheorem*{L1}{Lemma~\ref{lemma:cvf}}
\newtheorem*{C1}{Corollary~\ref{lemma:cvf:unbiased}}
\newtheorem*{C2}{Corollary~\ref{lemma:cvf:consistent}}
\label{appendix}

\subsection{Featurized KMM and KLIEP}
\label{dre_details}


Similar in spirit to the probabilistic classification approach in Section~\ref{existing-dre}, we note that it is quite straightforward to extend this technique to non-parametric density ratio estimation methods. Suppose that $(\hat{r}' \circ f_\theta)$ is obtained from $\hat{r}' = \dre \left(f_\theta(\Dc_p),f_\theta(\Dc_q) \right)$. Then, we find that the solution to KMM is equivalent after we first map the inputs into the feature space via $f_\theta: \mathcal{X} \rightarrow \mathcal{Z}$:
\begin{align*}
    \min_{\hat{r} \in \mathcal{H}}||\mathbb{E}_{q'(f_\theta(\bx))}\left[k(f_\theta(\bx),\cdot)(\hat{r}' \circ f_\theta)(\bx) \right] - \mathbb{E}_{p'(f_\theta(\bx))}\left[k(f_\theta(\bx),\cdot) \right]||^2_{\mathcal{H}}
\end{align*}

For KLIEP, we may also solve for the density ratio estimates in feature space:
\begin{align*}
    &\min_{(\hat{r}' \circ f_\theta)(\bx)} \mathbb{E}_{p'(f_\theta(\bx))} \left[ \log \frac{p'(f_\theta(\bx))}{(\hat{r}' \circ f_\theta)(\bx)q'(f_\theta(\bx))} \right]\\
    &\textrm{s.t. } \int (\hat{r}' \circ f_\theta)(\bx)q'(f_\theta(\bx)) d\bx = 1
\end{align*}
as a straightforward consequence of Lemma~\ref{lemma:cvf}.

\subsection{Derivations for Motivating Example}\label{intuition}
We derive the calculations from the simple example presented in Section~\ref{intuition_ex}, and restate the problem setting here for completeness. Suppose we want to estimate the density ratios between two Gaussians, $p \sim \mathcal{N}(m,1)$ and $q \sim \mathcal{N}(-m,1)$, with a finite number of samples $\Dc_p$ and $\Dc_q$ of size $n$ from each. We denote the random variable with the density $p$ as $X_p$, and the random variable with density $q$ as $X_q$. Our intuition was that as $m$ grows larger, the probability that we would observe all positive samples from $p$ (and analogously all negative samples from $q$) would be extremely high. 

Without loss of generality, we first compute $P(X_p \leq 0)$ by first using the well-known (lower) tail bound for Gaussian random variables:
\begin{align*}
    P(X_p \leq x) &\leq \inf_{\theta \geq 0} \exp(-\theta x)\psi(\theta)\\
    &= \inf_{\theta \geq 0} \exp(-\theta x + \theta m + \theta^2 /2) \\
    &= \exp(-m^2/2) \textrm{ for all } x \leq 0
\end{align*}
since the minimum is achieved at $\theta^*=x-m$, where $\psi(\theta) = \exp(\theta m + \theta^2/2)$ is the moment generating function for $\mathcal{N}(m,1)$. This tells us that the probability of observing a single positive sample from $p$ is $P(X_p > 0) = 1 - P(X_p \leq 0) \geq 1 - \exp(-m^2/2)$, so taking account the fact that we have $n$ i.i.d. samples gives us: $\prod_{i=1}^n P(X_p > 0) \geq (1 - \exp(-m^2/2))^n$.

Next, we compute the probability of seeing a single sample in our training set such that $X_p \leq 0$. Our reasoning was that such observed examples would help mitigate the pathological behavior of our learning algorithm driving up the magnitude of the logistic regression parameters to infinity. We find that:
\begin{align*}
    P(\textrm{at least one } X_p \leq 0) &= 1 - \prod_{i=1}^n P(X_p > 0) \\
    &\leq 1 - (1 - \exp(-m^2/2))^n
\end{align*}
which is an extremely low probability. In fact, if we set $1 - (1 - \exp(-m^2/2))^n = \delta$ and solve for $\delta$, we find that:
\begin{align*}
    1 - (1 - \exp(-m^2/2))^n &= \delta \\
    (1 - \exp(-m^2/2))^n &= 1 - \delta \\
    n \log (\exp(-m^2/2)) &= \log (1-\delta)\\
    n &= \frac{\log(1-\delta)}{\log(1-\exp(-m^2/2))}
\end{align*}
Therefore, we observe a non-positive sample from $p$ with probability at most $\delta$ for $n < \frac{\log(1-\delta)}{\log(1-\exp(-m^2/2))}$.

For a perhaps more intuitive bound, we can use Bernoulli's inequality, which states that $(1+x)^r \geq 1 + r \cdot x$ for $x \geq -1, r \in \mathbb{R}/\ (0,1)$. Doing so, we see that: 
\begin{align*}
\prod_{i=1}^n P(X_p > 0) &\geq (1 - \exp(-m^2/2))^n \\
&\geq (1 - n \cdot \exp(-m^2/2))
\end{align*}
which indicates that we require a training set size that is exponential in the order of $m^2$ to avoid the pathological scenario described in Section~\ref{intuition_ex}.

\subsection{Proofs for Theoretical Results}
\label{supp_proofs}
\subsubsection{Proof of Lemma 1}\label{lemma-proof}
For completeness, we restate Lemma~\ref{lemma:cvf} prior to providing the proof.
\begin{L1}
Let $X_p \sim p$ be a random variable with density $p$, and $X_q \sim q$ be a random variable with density $q$. Let $f_\theta$ be any invertible mapping. Let $p',q'$ be the densities of $Z_p = f_\theta(X_p)$ and $Z_q = f_\theta(X_q)$ respectively.
Then for any $\bx$:
\[
\frac{p(\bx)}{q(\bx)} =  \frac{p'(f_\theta(\bx))}{q'(f_\theta(\bx))}
\]
\end{L1}
\begin{proof}
By the change of variables formula:
\begin{align*}
\frac{p(\bx)}{q(\bx)} = \frac{p(\bx) \left|\left[\frac{\partial f_\theta^{-1}(\bt)}{\partial \bt} \right]_{t = f_\theta(\bx)}\right|}{q(\bx) \left|\left[\frac{\partial f_\theta^{-1}(\bt)}{\partial \bt} \right]_{t = f_\theta(\bx)}\right|} = \frac{p'(f_\theta(\bx))}{q'(f_\theta(\bx))} 
\end{align*}
\end{proof}

\subsubsection{Proof of Unbiasedness for the Featurized Density Ratio Estimator (Corollary 1)} \label{unbiased-supp}
For completeness, we restate Corollary~\ref{lemma:cvf:unbiased} prior to providing the proof.
\begin{C1}
Let $\Dc_p$ be $n$ i.i.d samples from density $p$, and $\Dc_q$ be $n$ i.i.d samples from density $q$.
Let $\hat{r}(\bx)$ obtained from $\hat{r} = \dre \left(\Dc_p,\Dc_q \right)$
be an unbiased estimator of $r(\bx) = \frac{p(\bx)}{q(\bx)}$ and any $p,q$, and let $f_\theta$ denote any invertible mapping. 
Then,
$(\hat{r}' \circ f_\theta)(\bx)$ obtained from $\hat{r}' = \dre \left(f_\theta(\Dc_p), f_\theta(\Dc_q) \right)$
is also an unbiased estimator of $\frac{p(\bx)}{q(\bx)}$ for any $p,q$.
\end{C1}
\begin{proof}
Using the definition of unbiasedness, we have:
\[
\mathbb{E}_{p(\bx),q(\bx)} \left[ \hat{r}(\bx) \right] = \frac{p(\bx)}{q(\bx)}
\]
Let $p',q'$ be the densities of $f_\theta(X_p)$ and $f_\theta(X_q)$, respectively. Consider the estimator $\hat{r}'= \dre(f_\theta(\mathcal{D}_p),f_\theta(\mathcal{D}_q))$ which is an unbiased estimator of $\frac{p'(f_\theta(\bx))}{q'(f_\theta(\bx))}$ by assumption. Then:
\[
\mathbb{E}_{p'(f_\theta(\bx)),q'(f_\theta(\bx))} \left[ (\hat{r}' \circ f_\theta)(\bx) \right] = \frac{p'(f_\theta(\bx))}{q'(f_\theta(\bx))}
\]
By the definition of $p',q'$, this is equivalent to:
\[
\mathbb{E}_{p(\bx),q(\bx)} \left[ (\hat{r}' \circ f_\theta)(\bx) \right] = 
\frac{p'(f_\theta(\bx))}{q'(f_\theta(\bx))} = \frac{p(\bx)}{q(\bx)}
\]
where the last equality follows from Lemma~\ref{lemma:cvf}. 
\end{proof}

\subsubsection{Proof of Consistency (Corollary 2)} \label{consistency-supp}
For completeness, we restate Corollary~\ref{lemma:cvf:consistent} before the proof statement.
\begin{C2}
Let $\Dc_p$ be $n$ i.i.d samples from density $p$, and $\Dc_q$ be $n$ i.i.d samples from density $q$. Let $\hat{r}(\bx)$ obtained from $\hat{r} = \dre(\Dc_p,\Dc_q)$
be a consistent estimator of $\frac{p(\bx)}{q(\bx)}$ for all $\bx \in \mathcal{X}$ and for any $p,q$. Additionally, let $f_\theta$ be any invertible mapping. Then, $(\hat{r}' \circ f_\theta)(\bx)$ obtained from $\hat{r}' = \dre \left(f_\theta(\Dc_p), f_\theta(\Dc_q)\right)$
is also a consistent estimator of $\frac{p(\bx)}{q(\bx)}$ for any $p,q$.
\end{C2}

\begin{proof}
By the definition of consistency, we have that $\forall \bx \in \mathcal{X}$ and $\forall \epsilon > 0$:
\[
\lim_{n \rightarrow \infty} P_{p,q} \left[
\left|\hat{r}(\bx)-\frac{p(\bx)}{q(\bx) } \right| > \epsilon\right] = 0
\]
Let $p',q'$ be the densities of $f_\theta(X_p)$ and $f_\theta(X_q)$ respectively. Because the estimator is assumed to be consistent for any $p,q$:
\[
\lim_{n \rightarrow \infty} P_{p',q'} \left[
\left|\hat{r}'(\bx)-\frac{p'(\bx)}{q'(\bx) } \right| > \epsilon\right] = 0
\]
and by definition of $p',q'$ this is equivalent to:
\[
\lim_{n \rightarrow \infty} P_{p,q} \left[
\left|\hat{r}'(\bx)-\frac{p'(\bx)}{q'(\bx) } \right| > \epsilon\right] = 0
\]
Because the condition holds $\forall \bx \in \mathcal{X}$, we have:
\[
\lim_{n \rightarrow \infty} P_{p,q} \left[
\left|(\hat{r}' \circ f_\theta)(\bx)-\frac{p'(f_\theta(\bx))}{q'(f_\theta(\bx)) } \right| > \epsilon\right] = 0
\]
\[
\lim_{n \rightarrow \infty} P_{p,q} \left[
\left|(\hat{r}' \circ f_\theta)(\bx)-\frac{p(\bx)}{q(\bx) } \right| > \epsilon\right] = 0
\]
where the last equality is due to Lemma~\ref{lemma:cvf}.
\end{proof}

\subsection{Additional Experimental Details}
\label{exp_details}
\subsubsection{Miscellaneous Background Information} \label{hacks}
\paragraph{Data Preprocessing.}
Prior to training the MAF, we: (a) use uniform dequantization; (b) rescale the pixels to lie within [0,1], and apply the logit transform following \citep{papamakarios2017masked}. For classification, we simply rescale the pixels to lie within [0,1].

\paragraph{Importance Weighting in Practice.}
As noted in \citep{grover2019bias}, we apply two techniques when using the learned density ratio estimates as importance weights in our experiments. 
\begin{enumerate}
    \item Self-normalization: As a way to reduce variance, we ensure that the importance weights in a batch of $n$ examples sum to 1, as in the expression below:
\[
\tilde{r}(\bx_i) = \frac{\hat r(\bx_i)}{\sum_{j=1}^n \hat r(\bx_j)}
\]
We find that this technique works quite well when estimating density ratios in input space.
    \item Flattening: we raise our obtained density ratio estimates to the power of a scaling parameter $\gamma \geq 0$:
    $$\tilde{r}(\bx_i) = \hat r(\bx_i)^\gamma$$
    Empirically, we observe that this approach works best on the ratios obtained in feature space.
\end{enumerate}

\subsubsection{KMM and KLIEP}\label{appendix:kmm-kliep-datasets}
\paragraph{Code.} For our experiments using KMM and KLIEP, we based our code on the following implementations:
\begin{itemize}
    \item \url{https://github.com/sksom/Classification-using-KMM-Kernel-Mean-Matching-}
    \item \url{https://github.com/srome/pykliep}
\end{itemize}
\paragraph{Datasets.} We used two datasets for both the KMM and KLIEP experiments: a generated 2D mixture of Gaussians dataset, and the Breast Cancer Wisconsin Data Set from the UCI Archive [\cite{Dua:2019}]. For each dataset, we construct our source and target splits as follows:
\begin{itemize}
    \item 2D Mixture of Gaussians: For our source dataset, we sampled $10$ points from $\mathcal{N}(0, I)$ and $990$ points from $\mathcal{N}([3,3]^T, I)$, and for our target dataset, we sampled 990 points from $\mathcal{N}(0, I)$ and 10 points from $\mathcal{N}([3,3]^T, I)$.
    \item Breast Cancer: Each sample consists of $9$ input features (each of which with values ranging from $0-9$) and one binary label. For each of $n=30$ trials, we first set aside $3/4$ of the dataset for our target dataset and then, with the remaining $1/4$ of the data, constructed a biased source dataset by subsampling the training data according to $P(s = 1|y = 1) = 0.1$ and $P(s = 1\mid y = -1) = 0.9$, where $s$ indicates whether or not we include the sample. After subsampling, we normalized each feature value to be mean $0$ and variance $1$ (the same as in \citep{huang2006correcting}).
    \item Blood Transfusion: This dataset consists of 748 samples (each corresponding to one person) with $5$ input features and one binary label that represents whether or not the person is a blood donor. For each of $n = 30$ trials, as with the Breast Cancer dataset, we set aside $3/4$ of the dataset for the target dataset and used the remaining $1/4$ of the data to construct a biased source dataset by subsampling $x_i$ according to $P(s_i \mid x_i) \propto \exp(-\sigma\lVert x_i - \bar{x}\rVert^2)$ where $\sigma = 1/20$ (following \citep{huang2006correcting}).
    \item Wine Quality: This dataset consists of 4898 samples with $12$ input features and a label between $0--10$ representing the wine quality. The binary classification task was the prediction of whether or not the wine quality was $\geq 5$. We followed the same subsampling setup as for the Blood Transfusion dataset.
\end{itemize} 
\paragraph{Models.}  For our KMM experiments on both the 2D Mixture of Gaussians and the Breast Cancer datasets, we did a grid search over two parameters: $\gamma$, the kernel width, and $B$, the upper bound on the density ratio estimates. We searched over the values $\gamma = \{0.01, 0.1, 0.5, 1.0\}$ and $B = \{1, 10, 100, 1000\}$.\\
For classification of the mixture of Gaussians, we used scikit-learn's \texttt{LogisticRegression} class. For the support vector classifier for the Breast Cancer dataset, we used scikit-learn's \texttt{SVC} class with a Gaussian kernel parameterized by $\gamma = 0.1$ penalty parameter $C = \{0.1, 1, 10, 100\}$ (the same setup as \citep{huang2006correcting}).

\subsubsection{Mutual Information Estimation}
For estimating MI, we follow the setup of \citep{belghazi2018mutual,poole2019variational,song2019understanding} but fix $\rho=0.9$. We generate a dataset of 100K examples, using a train/val/test split of 80K/10K/10K. 

\subsubsection{Targeted Generation with MNIST}
We note that a normalizing flow model that has been trained on \textit{any} mixture of $\mathcal{D}_p$ and $\mathcal{D}_q$ can be adapted for downstream applications of density ratio estimation. Concretely, we consider importance sampling, where we are interested in computing a statistic of the data $g(\cdot)$ with respect to a target distribution $p(\bx)$:
\begin{align*}
    \mathbb{E}_{p(\bx)}\left[ g(\bx) \right] &= \mathbb{E}_{ h(\bx)}\left[\frac{p(\bx)}{h(\bx)} g(\bx) \right]\\
    &= \mathbb{E}_{h(\bx)}\left[\frac{p(\bx)}{\frac{1}{2} (p(\bx) + q(\bx))} g(\bx) \right] \\
    &= \mathbb{E}_{h(\bx)}\left[\frac{r(\bx)}{\frac{1}{2} (r(\bx) + 1)} g(\bx) \right] \\
    &= \mathbb{E}_{h(\bx)}\left[r'(\bx) g(\bx) \right]
\end{align*}
where the flow has been trained on an equal-sized mixture of $\mathcal{D}_p$ and $\mathcal{D}_q$, the distribution learned by the flow is denoted as $h(\bx) = \frac{1}{2}p(\bx) + \frac{1}{2}q(\bx)$, and the importance weight (learned density ratio estimate) $r(\bx)$ has been re-balanced to account for the mixing proportions of $p(\bx)$ and $q(\bx)$ in the trained flow: $r'(\bx) = \frac{r(\bx)}{\frac{1}{2}(r(\bx)+1)}$. In the case that the mixing proportions are different (e.g. $\Dc_p$ and $\Dc_q$ are of different sizes), the re-balanced importance weight $r'(\bx)$ can be adjusted accordingly. We use this reweighting procedure in the MNIST targeted sampling experiments in Section~\ref{mnist_exp}.

After training our MAF model $f_\theta$ on the mixture of datasets $\mathcal{D}=\{\Dc_p, \Dc_q\}$, we use sampling-importance-resampling (SIR) \citep{liu1998sequential,doucet2000sequential,grover2019bias} to generate targeted samples from $q(\bx)$. Concretely, we sample $\mathbf{z}_1, ..., \mathbf{z}_n \sim t$ and compute  density ratio estimates $\hat r(\mathbf{z}_1), ..., \hat r(\mathbf{z}_n)$ with our trained probabilistic classifier $c_\phi$.
We then apply self-normalization as described in Appendix~\ref{hacks} to compute normalized importance weights $\tilde{r}(\mathbf{z}_1), ..., \tilde{r}(\mathbf{z}_n)$. Finally, we sample $j \sim$ Categorical$(\tilde{r}(\mathbf{z}_1), ..., \tilde{r}(\mathbf{z}_n))$ and generate our final sample $\mathbf{\hat x} = f_\theta^{-1}(\mathbf{ z}_j)$.

\subsubsection{Multi-class Classification with Omniglot}
For training the DAGAN, we followed \citep{antoniou2017data} and directly used the open-source implementation with default training parameters: batch size = 100, $\text{z\_dim}=100$, epochs = 200, 3 generator inner layers, 5 discriminator inner layers, a dropout rate value of 0.5, and the Adam optimizer with learning rate = $1e^{-4}$, $\beta_1=0$, and $\beta_2=0.9$. The repository can be found here: \url{https://github.com/AntreasAntoniou/DAGAN}. Following \citep{grover2019bias} and correspondence from the authors, we trained the DAGAN on the first 1200 character classes of Omniglot, which is typically used as the training split. Thus for both training the DAGAN and for the downstream classifier, we used the first 10 examples from the 1200 classes as the training set, the next 5 examples as the validation set, and the final 5 examples as the test set. All reported numbers in Table~\ref{table:omniglot_exp} are obtained on the final test set.

For the multi-class classification, we used the CNN-based architecture in \citep{vinyals2016matching} as shown in Table~\ref{table:fewshot_arch}. For data augmentation, we randomly sampled 50 examples for each of the 1200 classes from the trained DAGAN -- thus for all other models aside from the \texttt{Data-only} baseline, the training set size increased from (1200*10) to (1200*60).

For importance weighting, we trained both binary classifiers and input-space and feature-space to distinguish between the real and synthetic examples. We applied early stopping to the density ratio classifiers based on the validation set, which was comprised of 5 real examples and 5 synthetic examples. For the input-space density ratio estimation classifier, we found that the self-normalization technique worked best. For the feature-space density ratio estimation classifier, however, we found that flattening with $\gamma=0.2$ worked well, and used this configuration. Additional details on self-normalization and flattening can be found in Appendix~\ref{hacks}.

The importance weighting procedure when training the \textit{downstream classifier} was \textit{only} applied to the synthetic examples -- no additional reweighting was applied to the real examples. Additional details on hyperparameter configurations for both classifiers can be found in Appendix~\ref{dre_clf} and \ref{omniglot_clf}.

\subsection{Architecture and Hyperparameter Configurations}
\label{appendix_arch}

\subsubsection{Masked Autoregressive Flow (MAF)}
For the: (1) synthetic experiments with KMM/KLIEP; (2) toy 2-D Gaussian experiments; (3) mutual information estimation experiments; and (4) few-shot classification experiments with Omniglot, we leverage a Masked Autoregressive Flow (MAF) as our invertible generative model \citep{papamakarios2017masked}. The MAF is comprised of a set of MADE blocks \citep{germain2015made}, each with varying numbers of hidden layers and hidden units depending on the complexity of the dataset as shown in Table~\ref{table:maf_hyps}. We use the sequential input ordering with ReLU activations and batch normalization between the blocks. We build on top of a pre-existing PyTorch implementation (\texttt{https://github.com/kamenbliznashki/normalizing\_flows}).

\begin{table}[H]
\centering
\begin{tabular}{c|c|c|c|c}
\hline
\textbf{Dataset}& \textbf{n\_blocks} & \textbf{n\_hidden} & \textbf{hidden\_size} & \textbf{n\_epochs}\\
\hline
UCI + Synthetic & 5 & 1 & 100 & 100\\
\hline
Toy 2-D Gaussians & 5 & 1 & 100 & 100 \\
\hline
MI Gaussians & 5 & 1 & 100 & 200 \\
\hline
MNIST & 5 & 1 & 1024 & 200 \\
\hline
Omniglot & 5 & 2 & 1024 & 200\\
\hline
\end{tabular}
\caption{Configuration of the number of MADE blocks, number of hidden layers in each MADE block, the number of hidden units, and the total number of training epochs for each dataset.}
\label{table:maf_hyps}
\end{table}

\paragraph{Hyperparameters.} During training, we use a batch size of 100 and the PyTorch default values of the Adam optimizer with learning rate = 0.0001 and weight decay of 1e-6 for all datasets. We use early stopping on the best log-likelihood on a held-out validation set.


\subsubsection{MLP Classifier}
We utilize the following MLP classifier architecture as shown in Table~\ref{table:mlp_arch} for several of our experiments: (a) the synthetic 2-D Gaussians setup in Section~\ref{method}; (b) the mutual information estimation experiment; and (c) the attribute classifier for the targeted MNIST generation task.
\begin{table}[H]
\centering
\begin{tabular}{c|c}
\hline
\textbf{Name}& \textbf{Component}\\
\hline
Input Layer & Linear $\texttt{in\_dim} \rightarrow \texttt{h\_dim}$, ReLU \\
\hline
Hidden Layer \#1 & Linear $\texttt{h\_dim} \rightarrow \texttt{h\_dim}$, ReLU \\
\hline
Hidden Layer \#2 & Linear $\texttt{h\_dim} \rightarrow \texttt{h\_dim}$, ReLU \\
\hline
Output Layer & Linear $\texttt{h\_dim} \rightarrow \texttt{out\_dim}$ \\
\hline
\end{tabular}
\caption{MLP classifier architecture.}
\label{table:mlp_arch}
\end{table}

\paragraph{Hyperparameters.} The relevant hyperparameters for the three previously mentioned experiments are shown in Table~\ref{table:mlp_hyps}. All experiments used the default values of the Adam optimizer unless otherwise specified, and employed early stopping on the best loss on a held-out validation set.
\begin{table}[H]
\centering
\begin{tabular}{c|c|c|c|c|c|c|c|c}
\hline
\textbf{Dataset} & \textbf{in\_dim} & \textbf{h\_dim} & \textbf{out\_dim} & \textbf{n\_epochs} & \textbf{batch\_size} & \textbf{learning\_rate} & \textbf{weight\_decay}\\
\hline
Toy 2-D Gaussians & 2 & 100 & 1 & 100 & 128 & 0.0002 & 0.0005 \\
\hline
MI Gaussians & 40 & 200 & 1 & 200 & 128 & 0.0002 & 0.0005  \\
\hline
MNIST & 784 & 100 & 1 & 10 & 128 & 0.0002 & 0.000  \\
\hline
\end{tabular}
\caption{Configuration of the MLP dimensions for each of the synthetic 2-D Gaussian, mutual information estimation, and MNIST attribute classification experiments, as well as several additional hyperparameters for training.}
\label{table:mlp_hyps}
\end{table}
\pagebreak
We note that for the attribute classifier for MNIST, we explored two scenarios: 
\begin{itemize}
    \item \texttt{diff-digits}, where all digits of classes \{1,2\} were given the label $y=0$, and digits of classes \{0,7\} were labeled as $y=1$
    \item \texttt{diff-background}, where all digits from the original dataset were labeled as $y=0$ and those with flipped colors (white background, black digits) were labeled as $y=1$.
\end{itemize}
In order to distinguish the separate classes for targeted generation, an MLP-based classifier was trained for each of the \texttt{diff-digits} and \texttt{diff-background} tasks as outlined in Tables~\ref{table:mlp_arch} and ~\ref{table:mlp_hyps}.

\subsubsection{Density Ratio Classifier} \label{dre_clf}
Depending on the complexity of the dataset, we used either an MLP classifier (Table~\ref{table:mlp_arch}) or CNN-based classifier (Table~\ref{table:fewshot_arch}) for the density ratio estimator. For all synthetic experiments including those conducted on the MNIST dataset, we used an MLP for both input-space and feature-space density ratio estimation. For the Omniglot experiments, we used a slightly modified version of the CNN-based classifier where we swap the final output layer to be a Linear layer of dimension $64 \rightarrow 1$.

\paragraph{Hyperparameters.} During training, we use a batch size of 64 and the Adam optimizer with learning rate = 0.001. The classifiers learn relatively quickly for both scenarios and we only needed to train for 10 epochs. 

\subsubsection{Downstream Classifier for Omniglot} \label{omniglot_clf}
For the multi-class classification task with Omniglot, we leveraged a commonly-used CNN architecture following \citep{vinyals2016matching}, as shown in the following table:

\begin{table}[h!]
\centering
\begin{tabular}{c|c}
\hline
\textbf{Name}& \textbf{Component}\\
\hline
conv1 & $3\times3$ conv, 64 filters, stride 1, BatchNorm2d, ReLU, $2\times2$ MaxPool \\
\hline
conv2 & $3\times3$ conv, 64 filters, stride 1, BatchNorm2d, ReLU, $2\times2$ MaxPool \\
\hline
conv3 & $3\times3$ conv, 64 filters, stride 1, BatchNorm2d, ReLU, $2\times2$ MaxPool \\
\hline
conv4 & $3\times3$ conv, 64 filters, stride 1, BatchNorm2d, ReLU, $2\times2$ MaxPool \\
\hline
Output Layer & Linear $64 \rightarrow 1200$, Softmax \\
\hline
\end{tabular}
\caption{CNN architecture for Omniglot experiments.}
\label{table:fewshot_arch}
\end{table}


\paragraph{Hyperparameters.} During training, we sweep over batch sizes of \{32,64,128\} and the Adam optimizer with learning rate = 0.001. We also swept over the flattening coefficient for density ratio estimation and found that $\gamma=0.2$ worked best. We trained the classifier for 100 epochs, and used early stopping on the validation set of Omniglot to determine the best model for downstream evaluation.

\begin{figure}
    \centering
        \subfigure[Scatter plot of $p(\bx)$ and $q(\bx)$]{\includegraphics[width=.33\textwidth]{figures/toy_gmm/toy_gmm_data.png}}
        \subfigure[Scatter plot colored by $\log r(\bx)$]{\includegraphics[width=.33\textwidth]{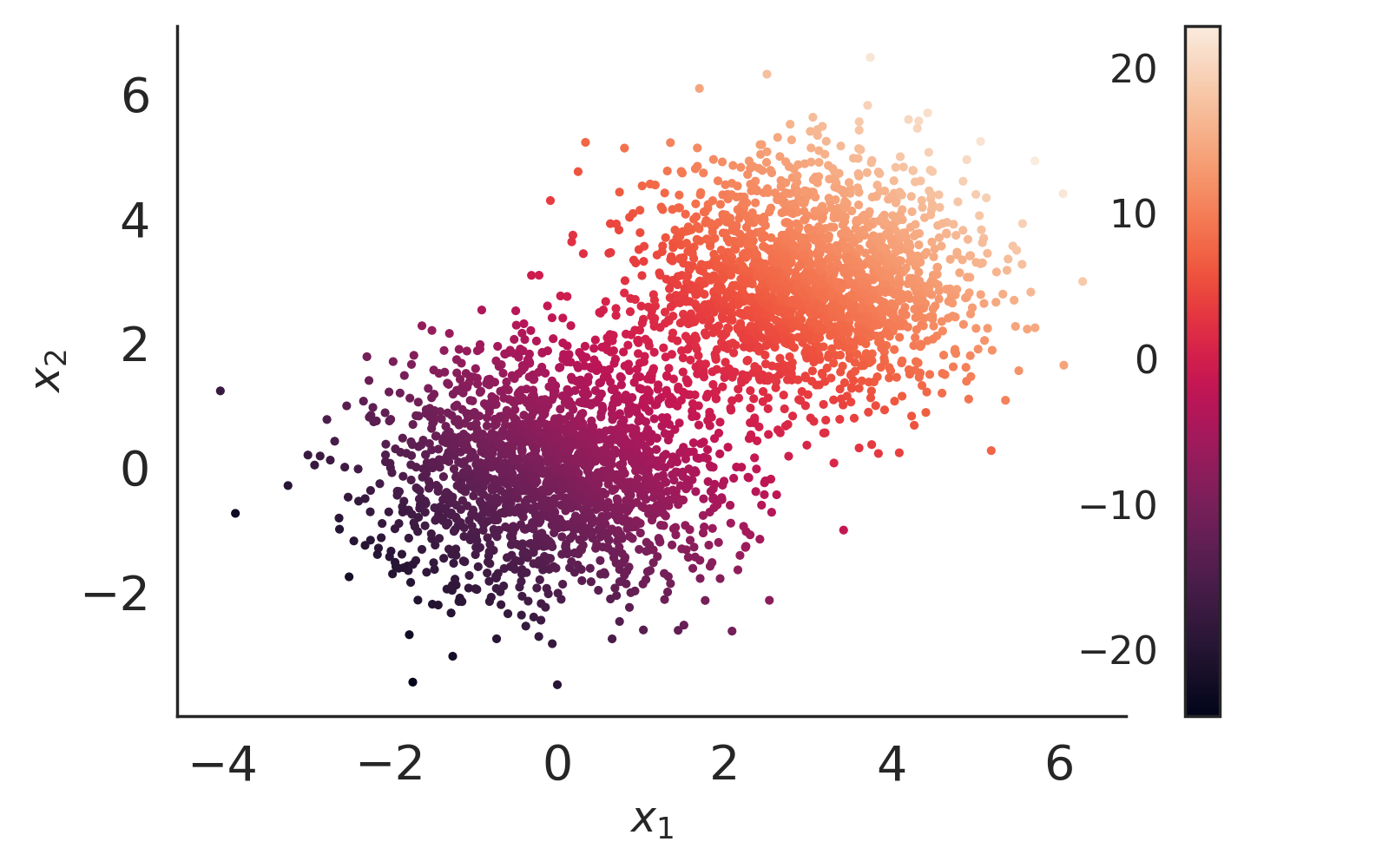}}
        \subfigure[Histogram of $\log r(\bx)$]{\includegraphics[width=.33\textwidth]{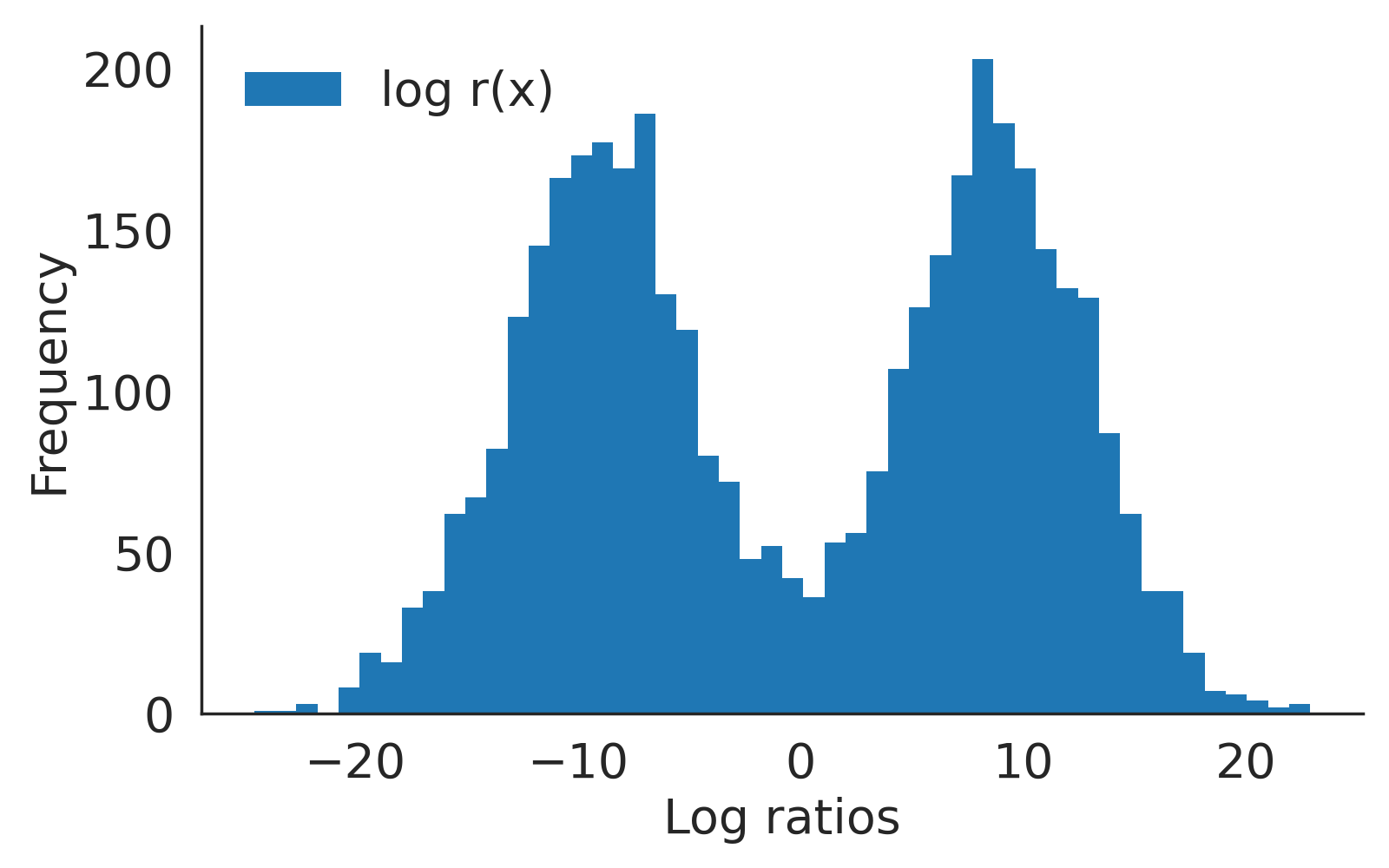}}
    \caption{(a) Data sampled from $p(\bx) \sim \mathcal{N}([0,0]^T, I)$ and $q(\bx) \sim \mathcal{N}([3,3]^T, I)$. (b) The same scatter plot as (a), but colored by the magnitude of the log density ratios. (c) Histogram of the log density ratios for each point in the dataset.}
    \label{fig:toy_gmm_data_supp}
\end{figure}

\begin{figure*}[!ht]
    \centering
        \subfigure[Joint training ($\alpha=0.01$)]{\includegraphics[width=.24\textwidth]{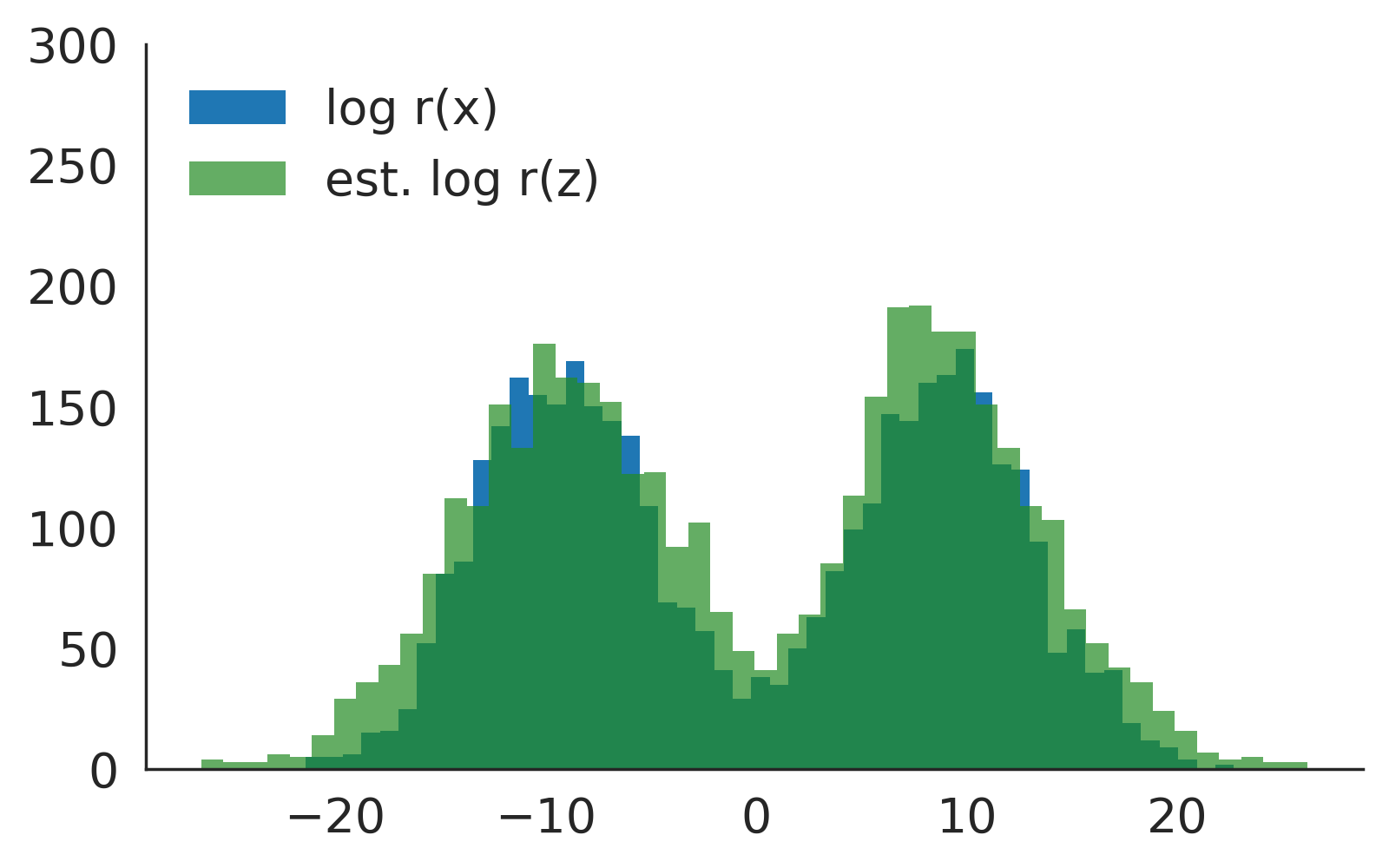}}
        \subfigure[Joint training ($\alpha=0.1$)]{\includegraphics[width=.24\textwidth]{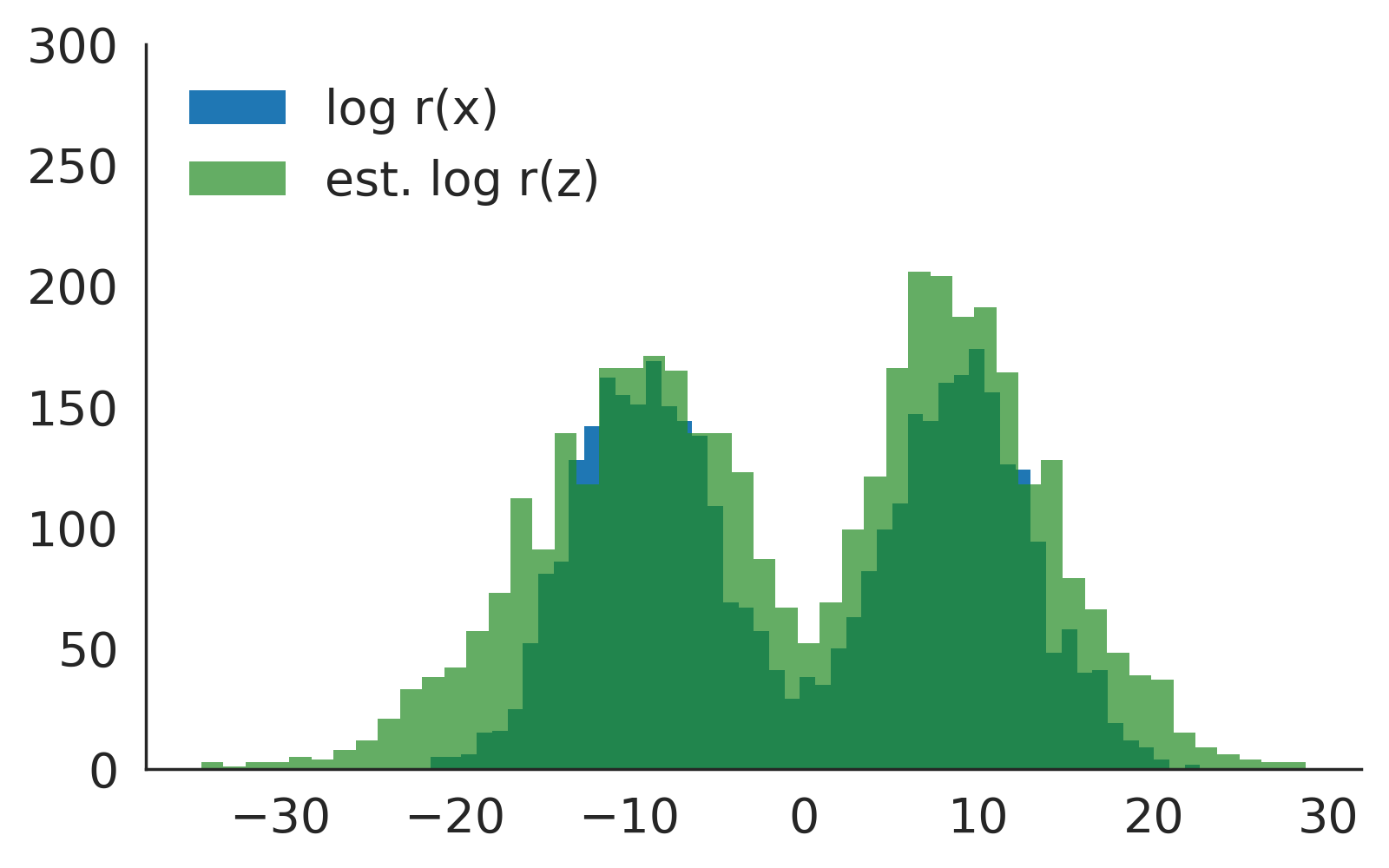}}
        \subfigure[Joint training ($\alpha=0.5$)]{\includegraphics[width=.24\textwidth]{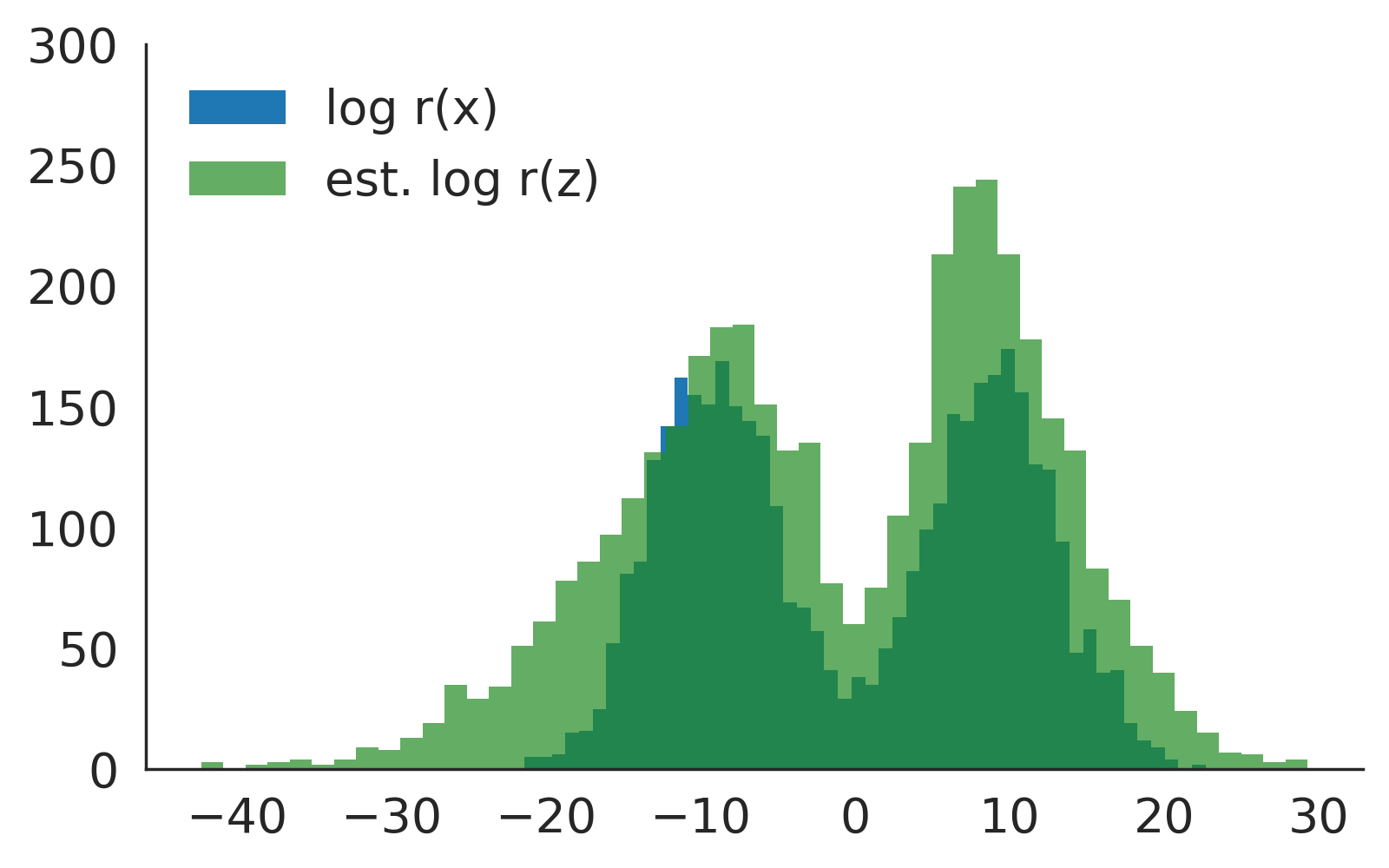}}
        \subfigure[Joint training ($\alpha=0.7$)]{\includegraphics[width=.24\textwidth]{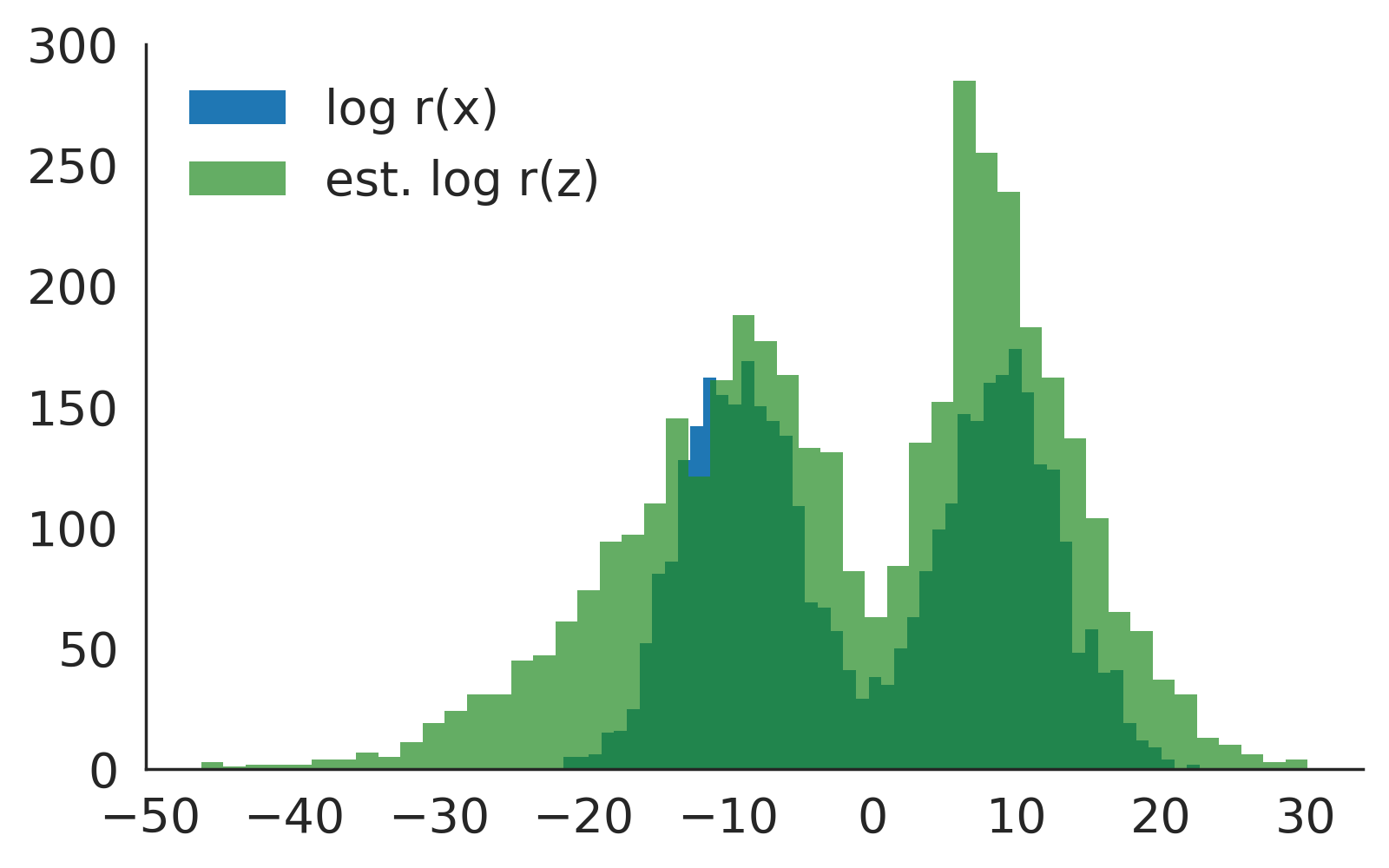}}
        \subfigure[Joint training ($\alpha=0.01$)]{\includegraphics[width=.24\textwidth]{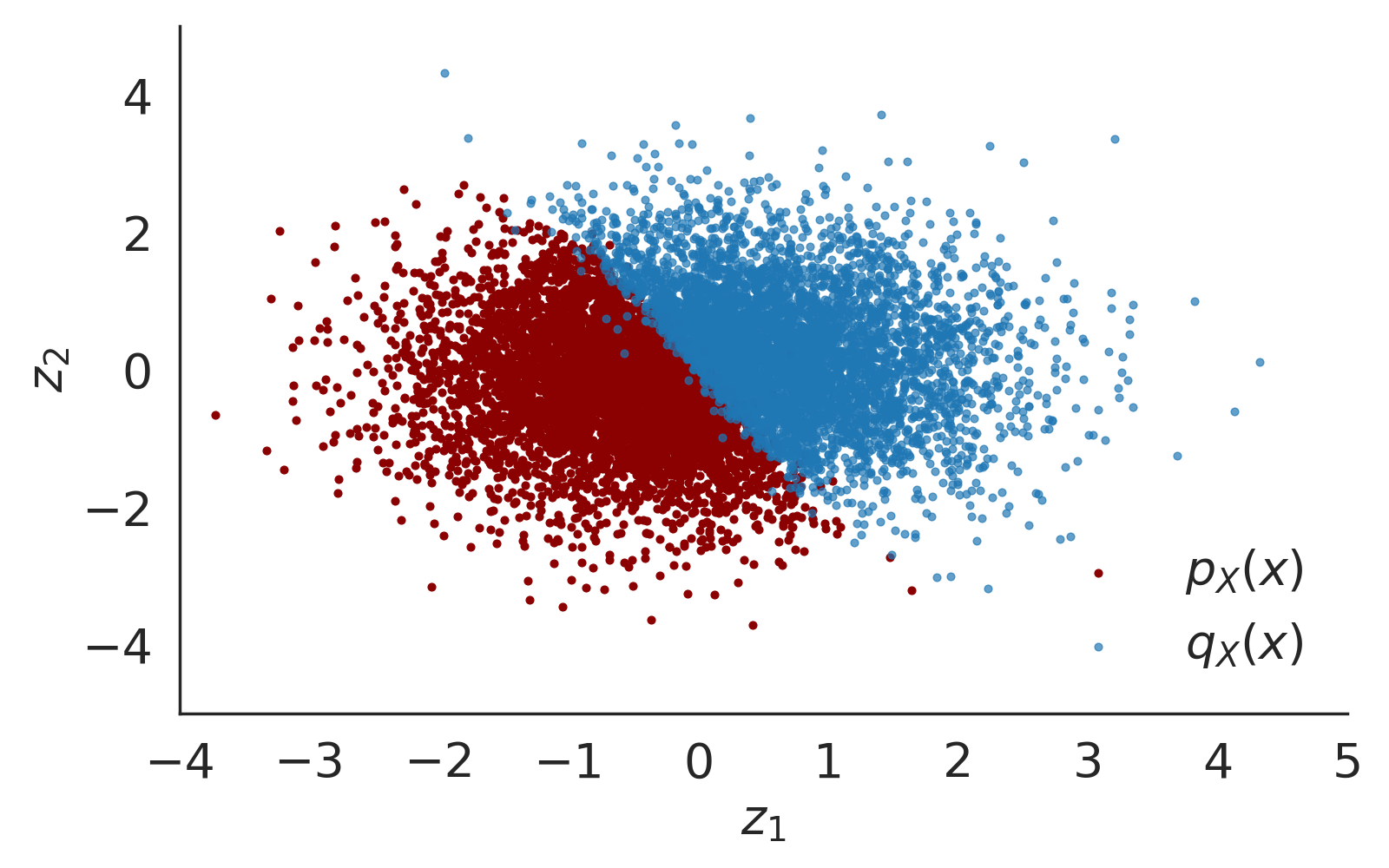}}
        \subfigure[Joint training ($\alpha=0.1$)]{\includegraphics[width=.24\textwidth]{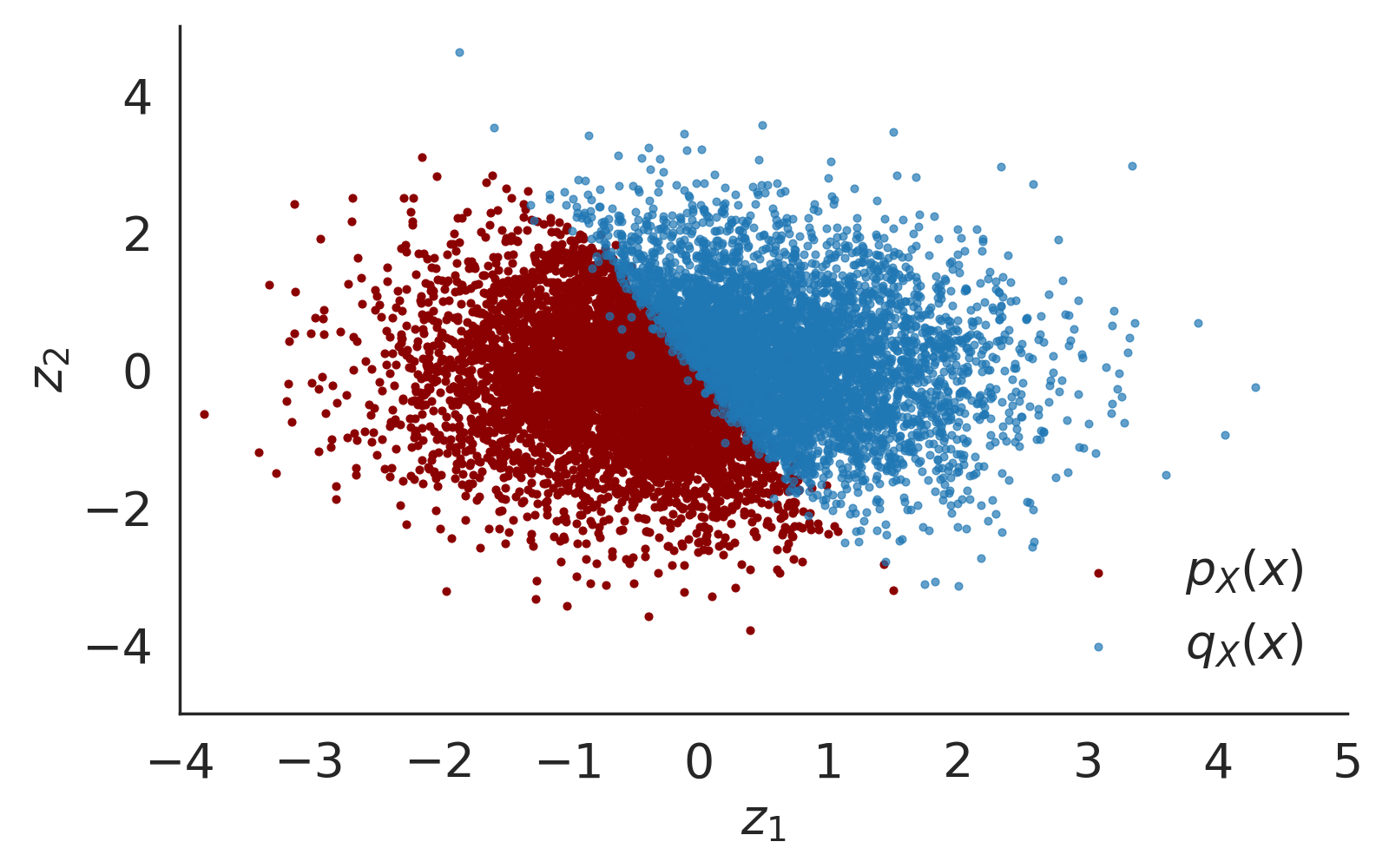}}
        \subfigure[Joint training ($\alpha=0.5$)]{\includegraphics[width=.24\textwidth]{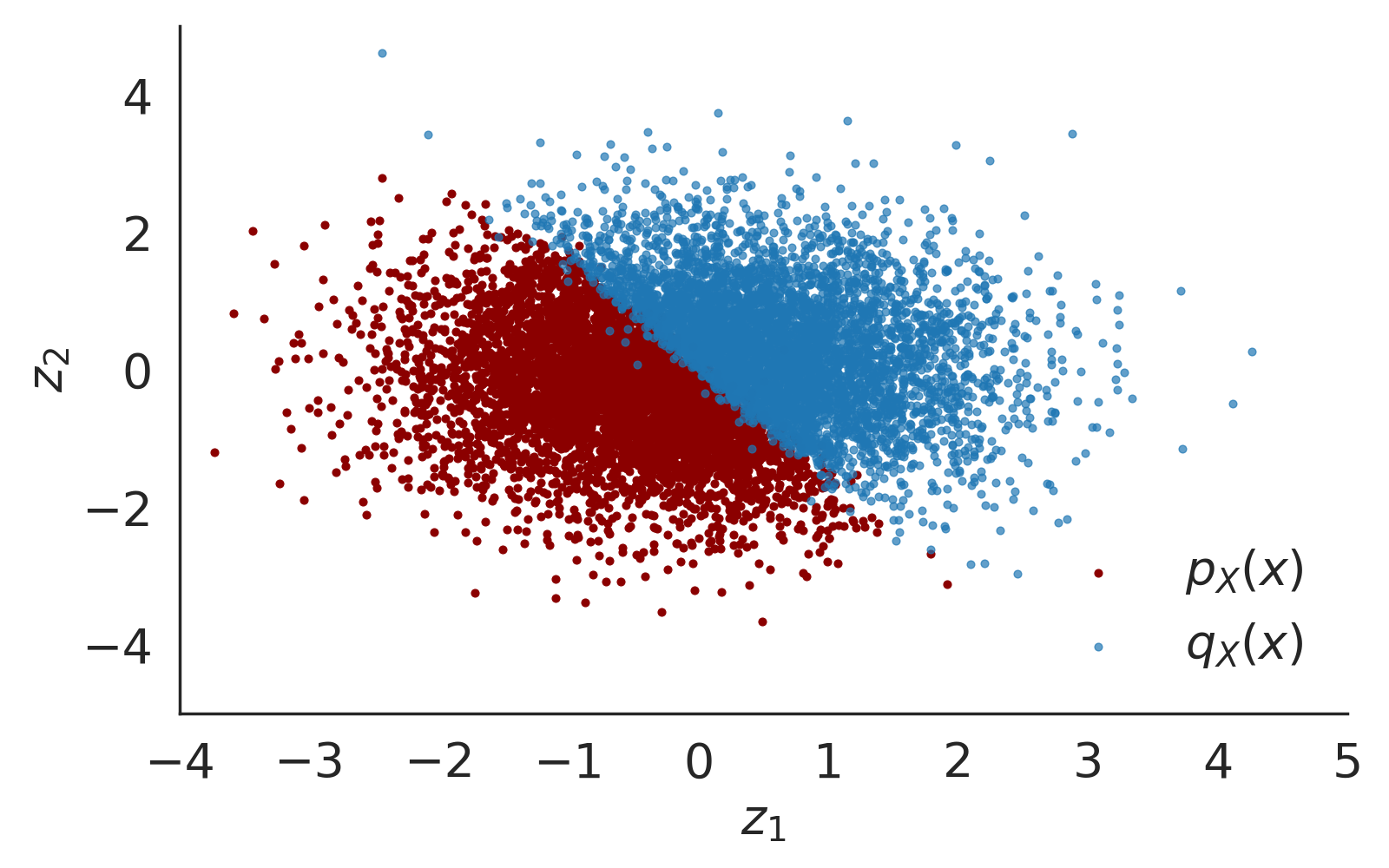}}
        \subfigure[Joint training ($\alpha=0.7$)]{\includegraphics[width=.24\textwidth]{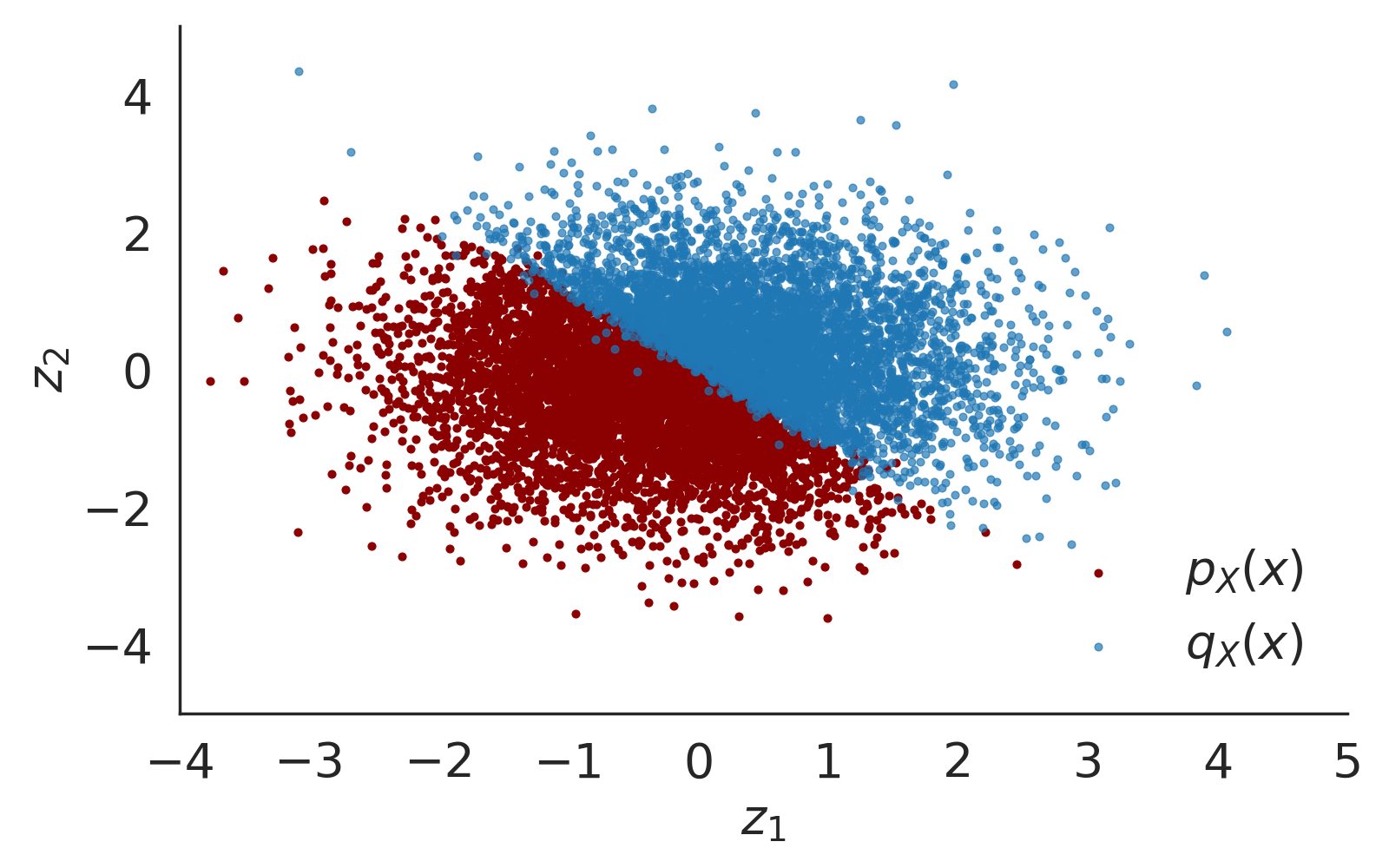}}
    \caption{\textit{Top row:} Additional results on the motivating example on a synthetic 2-D Gaussian dataset, with learned density ratio estimates by method relative to the ground truth values for (a-d).
    \textit{Bottom row:} Visualizations of the learned encodings for various training strategies for (e-h). We note that the jointly trained flow with the smallest value of $\alpha=0.01$ performs the best out of $\alpha=\{0.01, 0.1, 0.5, 0.7\}$.}
    \label{fig:gmm_results_supp}
\end{figure*}

\subsection{Additional Experimental Results}
\label{addtl-results}
\subsubsection{Toy Gaussian Mixture Experiment}
We provide additional experimental results on the motivating 2-D Gaussian mixture example introduced in Section~\ref{method}, where we sweep through additional values of $\alpha=\{0.01, 0.1, 0.5, 0.7\}$ on top of the one explored in the main text ($\alpha=0.9$). For reference, Figure~\ref{fig:toy_gmm_data_supp} displays the (a) ground truth data and log density ratios (b-c) that we hope to learn from samples. Results are shown in Figure~\ref{fig:gmm_results_supp}. A visual inspection of the 4 joint training procedures demonstrates that for this experiment, the jointly trained flow with the smallest contribution of the classification loss ($\alpha=0.01$ in (a)) outperforms all other methods (b-d). The learned feature space most closely resembles that of the separately trained flow in Figure~\ref{fig:gmm_encodings}(f), while the boundary separating the two densities $p$ and $q$ for the other models are skewed more to the left. 

\subsubsection{2-D Mixture of Gaussians for Featurized KLIEP/KMM} \label{kernel_synth}
In this experiment, we construct a synthetic domain adaptation task using 2-D Gaussian mixtures. Our goal is to assess whether our featurized density ratio estimation framework improves the performance of KMM and KLIEP, which operate in input space. We construct our source dataset as $\mathcal{D}_p \sim p(\bx) = 0.01 \cdot \mathcal{N}([0,0]^T, I) + 0.99 \cdot \mathcal{N}([3,3]^T, I)$, and our target dataset as $\mathcal{D}_q \sim q(\bx) = 0.99 \cdot \mathcal{N}([0,0]^T, I) + 0.01 \cdot \mathcal{N}([3,3]^T, I)$, where both datasets have $n=1000$ samples. We label samples from $\mathcal{N}([0,0]^T, I)$ as $y=1$ and samples from $\mathcal{N}([3,3]^T, I)$ as $y=0$. Then, we train a logistic regression classifier to distinguish between the two classes using 3 methods: 1) an unweighted logistic regression baseline, 2) reweighted logistic regression with importance weights computed in input space, and 3) reweighted logistic regression with importance weights computed in feature space. The importance weights are learned on a mixture of the source and target datasets. 

Results are shown in Table~\ref{table:synthetic_exp}.
\begin{table}[ht!]
\centering
\begin{tabular}{l|c|c}
\toprule
\textbf{Method} & \textbf{KMM} & \textbf{KLIEP} \\
\midrule
Unweighted logistic regression baseline & 0.236 $\pm$ 0.0456  & 0.236 $\pm$ 0.0456 \\ 
Logistic regression + IW(x) & 0.163 $\pm$ 0.0615  & 0.163 $\pm$ 0.0548 \\ 
Logistic regression + IW(z) (ours) & \textbf{0.0408} $\pm$ \textbf{0.0443}   & \textbf{0.125} $\pm$ \textbf{0.0269}\\ 
 \bottomrule
\end{tabular}
\caption{Comparison between test errors for unweighted logistic regression and reweighted x-space and z-space logistic regression on the 2-D Mixture of Gaussians dataset. Lower is better. Standard errors were computed over $10$ runs.}
\label{table:synthetic_exp}
\end{table}

\subsubsection{Domain Adaptation with the UCI Breast Cancer Dataset}
\label{appendix:uci}
We provide full experimental results of our domain adaptation experiment with the UCI Breast Cancer dataset in Table~\ref{table:uci} and Figure~\ref{fig:kliep_uci}. Results were computed over $30$ runs. We note that our method improves upon KMM for most values of $C$ and achieves the best absolute test error out of all combinations of $C$ with different methods. We also note that KLIEP performs poorly on this task, regardless of the method we use. 
\begin{table}[H]
\centering
\begin{tabular}{l|c|c|c|c}
\toprule
\textbf{KMM} & $C$=0.1 & $C$=1 & $C$=10 & $C$=100 \\
\midrule
Unweighted baseline &0.616 $\pm$ 0.0940  & 0.537 $\pm$ 0.167 & 0.591 $\pm$ 0.104 & 0.587 $\pm$ 0.114 \\ 
IW(x) & \textbf{0.596 $\pm$ 0.116} & 0.532 $\pm$ 0.198 & 0.577 $\pm$ 0.120 & 0.576 $\pm$ 0.118\\ 
IW(z) (ours) & 0.630 $\pm$ 0.0766   & \textbf{0.418} $\pm$ \textbf{0.221} & \textbf{0.421 $\pm$ 0.232} & \textbf{0.424 $\pm$ 0.230} \\
\midrule
\textbf{KLIEP} & $C$=0.1 & $C$=1 & $C$=10 & $C$=100 \\
\midrule
Unweighted baseline &0.616 $\pm$ 0.0940  & \textbf{0.537 $\pm$ 0.167} & 0.591 $\pm$ 0.104 & 0.587 $\pm$ 0.115 \\ 
IW(x) & \textbf{0.519 $\pm$ 0.214} & 0.589 $\pm$ 0.121 & \textbf{0.588 $\pm$ 0.114} & 0.587 $\pm$ 0.115\\ 
IW(z) (ours) & 0.650 $\pm$ 0.0109   & 0.55 $\pm$ 0.177 & 0.590 $\pm$ 0.126 & \textbf{0.586 $\pm$ 0.119} \\
\bottomrule
\end{tabular}
\caption{Comparison between test errors of an unweighted SVM and reweighted x-space and z-space SVMs on classification of the UCI Breast Cancer dataset with the biased class label sampling scheme. Standard errors were computed over $30$ runs.}
\label{table:uci}
\end{table}
\begin{figure}
    \centering
    \includegraphics[width=0.45\textwidth]{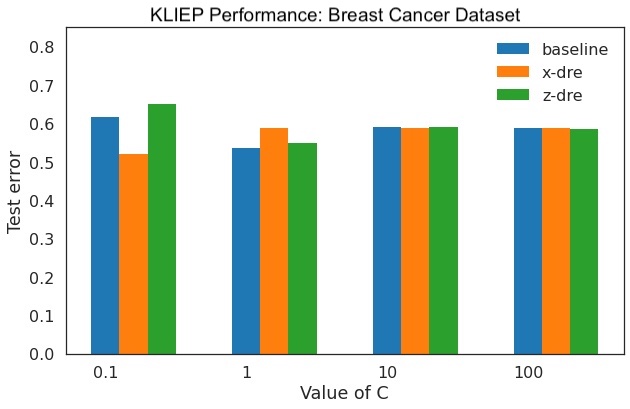}
    \caption{KLIEP test error of each method on binary classification of the UCI Breast Cancer dataset using a SVM parameterized by varying values of $C$. Lower is better. Results are averaged over 30 runs.}
    \label{fig:kliep_uci}
\end{figure}

\subsubsection{Omniglot samples from DAGAN}
A sample of $n=100$ examples synthesized by the trained DAGAN, used for the data augmentation experiments in Section~\ref{experiments}, are shown in Figure~\ref{fig:omni_samples}.

\subsubsection{Mutual Information Estimation}
In Figure~\ref{fig:mi_v2}, we replicate Figure~\ref{fig:mi} with additional results from joint training procedures using different values of $\alpha=\{0.1, 0.5, 0.9\}$ in Equation~\ref{joint_obj}. Specifically, we note that $\alpha=0.9$ outperforms all other jointly-trained models, indicating that a greater emphasis on the classification loss term helps for this experiment.
\subsubsection{Samples from MNIST Targeted Generation Task}
For each DRE in z-space, DRE in x-space, and unweighted settings and for \texttt{perc}=$\{0.1, 0.25, 0.5, 1.0\}$, Figures \ref{fig:mnist_diff_bkgd_fairgen_z}, \ref{fig:mnist_diff_bkgd_fairgen_x}, and \ref{fig:mnist_diff_bkgd_reg_gen} show $n = 100$ MAF-generated samples from the \texttt{diff-background} experiments and Figures \ref{fig:mnist_diff_digits_fairgen_z}, \ref{fig:mnist_diff_digits_fairgen_x}, and \ref{fig:mnist_diff_digits_reg_gen}, show $n = 100$ MAF-generated samples from the \texttt{diff-digits} experiments.

\begin{figure}[t]
    \centering
        \includegraphics[width=.23\textwidth]{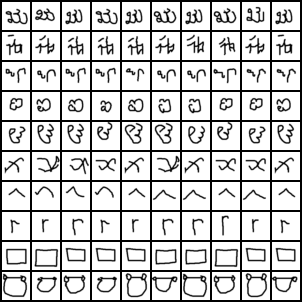}
    \caption{Generated samples from trained DAGAN, which are used as synthetic examples for data augmentation in the downstream Omniglot classification experiment.}
    \label{fig:omni_samples}
   \end{figure}

\begin{figure}[t]
    \centering
        \includegraphics[width=.65\textwidth]{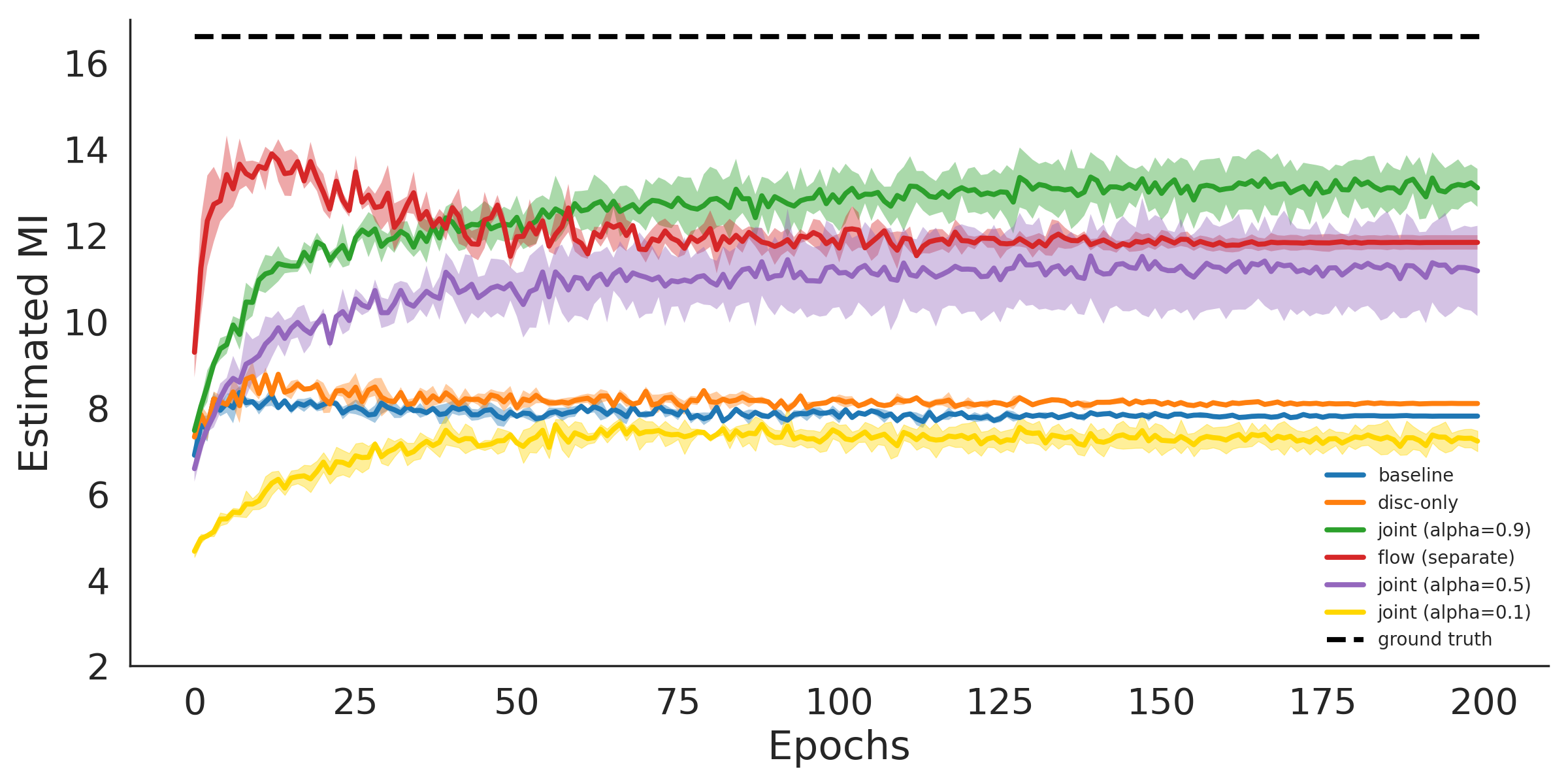}
    \caption{Estimated MI for the various training strategies. The true MI for the corresponding value of $\rho=0.9$ is $16.67$. While the separate training method outperforms all baselines, we note that joint training also achieves competitive performance with larger values of $\alpha$.}
    \label{fig:mi_v2}
   \end{figure}
\pagebreak
\begin{figure}[t]
        \centering %
        \subfigure[perc=0.1]{\includegraphics[width=.2\textwidth]{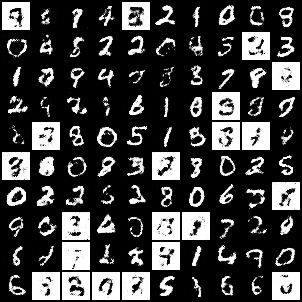}}
        \subfigure[perc=0.25]{\includegraphics[width=.2\textwidth]{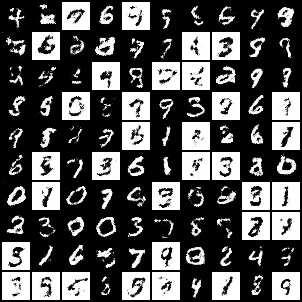}}
        \subfigure[perc=0.5]{\includegraphics[width=.2\textwidth]{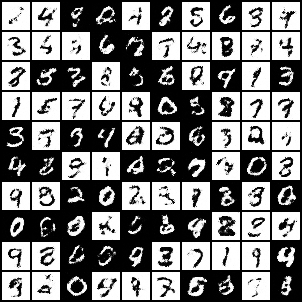}}
        \subfigure[perc=1.0]{\includegraphics[width=.2\textwidth]{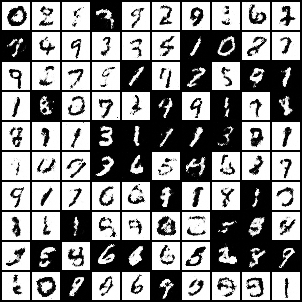}}
    \caption{SIR sampling with DRE in z-space}\label{fig:mnist_diff_bkgd_fairgen_z}
   \end{figure}
   
   \begin{figure}[t]
        \centering %
        \subfigure[perc=0.1]{\includegraphics[width=.2\textwidth]{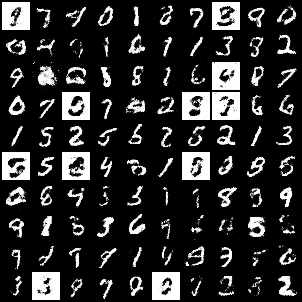}}
        \subfigure[perc=0.25]{\includegraphics[width=.2\textwidth]{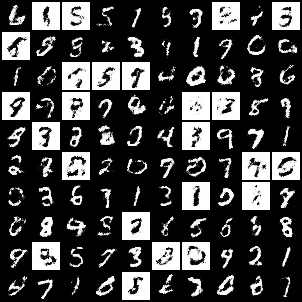}}
        \subfigure[perc=0.5]{\includegraphics[width=.2\textwidth]{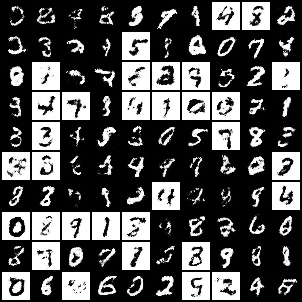}}
        \subfigure[perc=1.0]{\includegraphics[width=.2\textwidth]{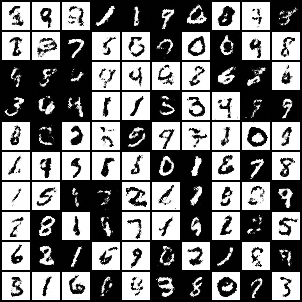}}
    \caption{SIR sampling with DRE in x-space}\label{fig:mnist_diff_bkgd_fairgen_x}
   \end{figure}
  \begin{figure}[t]
        \centering %
        \subfigure[perc=0.1]{\includegraphics[width=.2\textwidth]{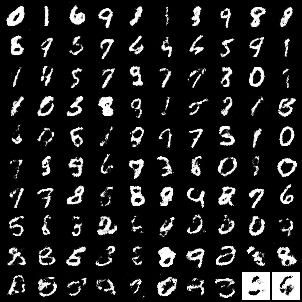}}
        \subfigure[perc=0.25]{\includegraphics[width=.2\textwidth]{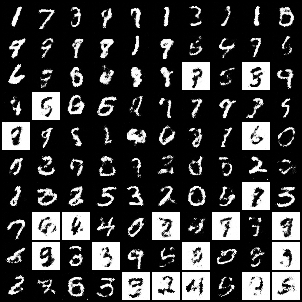}}
        \subfigure[perc=0.5]{\includegraphics[width=.2\textwidth]{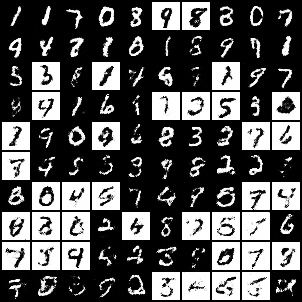}}
        \subfigure[perc=1.0]{\includegraphics[width=.2\textwidth]{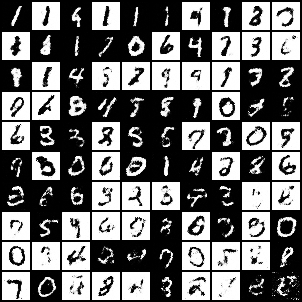}}
    \caption{Regular sampling}\label{fig:mnist_diff_bkgd_reg_gen}
   \end{figure}

\begin{figure}[t]
        \centering %
        \subfigure[perc=0.1]{\includegraphics[width=.2\textwidth]{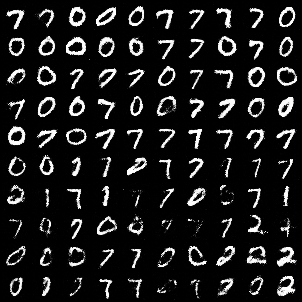}}
        \subfigure[perc=0.25]{\includegraphics[width=.2\textwidth]{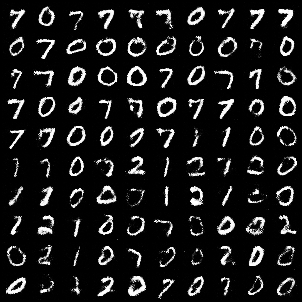}}
        \subfigure[perc=0.5]{\includegraphics[width=.2\textwidth]{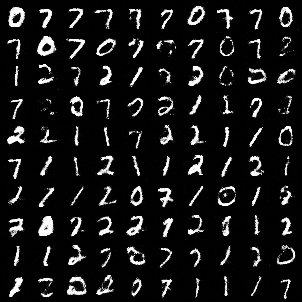}}
        \subfigure[perc=1.0]{\includegraphics[width=.2\textwidth]{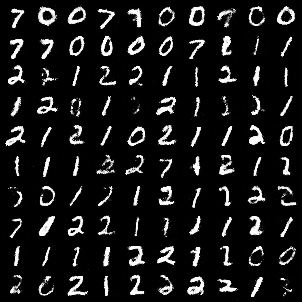}}
    \caption{SIR sampling with DRE in z-space}\label{fig:mnist_diff_digits_fairgen_z}
   \end{figure}
   \begin{figure}[t]
        \centering %
        \subfigure[perc=0.1]{\includegraphics[width=.2\textwidth]{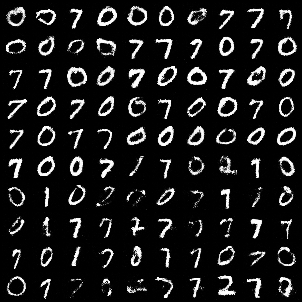}}
        \subfigure[perc=0.25]{\includegraphics[width=.2\textwidth]{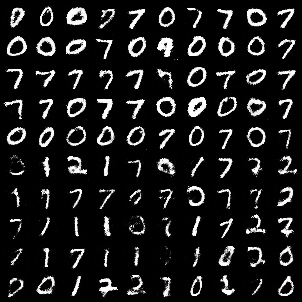}}
        \subfigure[perc=0.5]{\includegraphics[width=.2\textwidth]{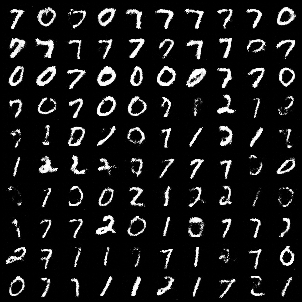}}
        \subfigure[perc=1.0]{\includegraphics[width=.2\textwidth]{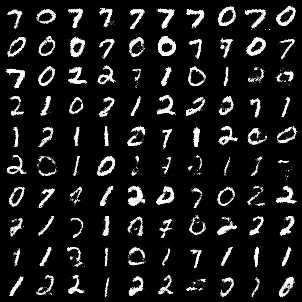}}
    \caption{SIR sampling with DRE in x-space}\label{fig:mnist_diff_digits_fairgen_x}
   \end{figure}
  \begin{figure}[t]
        \centering %
        \subfigure[perc=0.1]{\includegraphics[width=.2\textwidth]{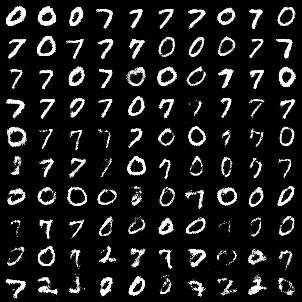}}
        \subfigure[perc=0.25]{\includegraphics[width=.2\textwidth]{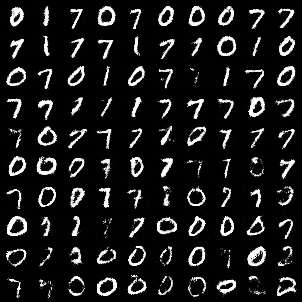}}
        \subfigure[perc=0.5]{\includegraphics[width=.2\textwidth]{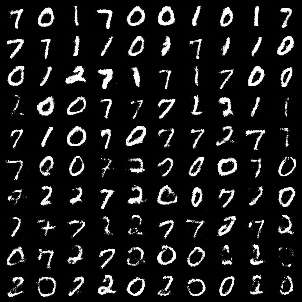}}
        \subfigure[perc=1.0]{\includegraphics[width=.2\textwidth]{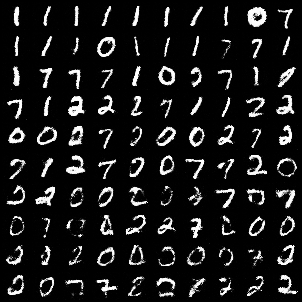}}
    \caption{Regular sampling}\label{fig:mnist_diff_digits_reg_gen}
   \end{figure}

\end{document}